\documentclass[10pt]{article}
\usepackage{amsmath}
\usepackage{amssymb}
\usepackage{amsthm}
\PassOptionsToPackage{x11names}{xcolor}
\usepackage[most]{tcolorbox}
\usepackage[preprint]{tmlr}

\usepackage{xspace}
\usepackage{ragged2e}

\DeclareRobustCommand{\refwarmstartassump}{\ref{assumption:warm-start}}
\DeclareRobustCommand{\refdelayedsec}{\ref{sec:apple-tasting-one-sided-delayed}}

\usepackage{amsmath}
\usepackage{amsfonts}
\usepackage{amsthm}
\usepackage{smile}
\usepackage{subfigure}
\usepackage{wrapfig}

\usepackage{tablefootnote}

\usepackage{hyperref}
\hypersetup{
  colorlinks=true,
  linkcolor=tmlrBlueDark,
  citecolor=tmlrCiteBlue,
  urlcolor=tmlrBlueDark
}
\usepackage{enumitem}
\usepackage{mathtools}
\usepackage{xcolor}
\definecolor{LightCyan}{rgb}{0.5, 0.65, 1}

\usepackage{colortbl}
\definecolor{LightCyan}{rgb}{0.8, 0.9, 1}

\usepackage{amssymb}
\usepackage{pifont}
\usepackage{fontawesome5}

\definecolor{snowblue}{RGB}{70, 140, 200}
\definecolor{flamered}{RGB}{220, 80, 50}

\newcommand{\methodcoldstart}{\text{Cold-Start UCB}}
\newcommand{\methodpretrained}{\text{Pre-trained Scorer}}

\newcommand{\actionadd}{\mathsf{add}}
\newcommand{\actionskip}{\mathsf{skip}}

\newcommand{\regret}{{\mathrm{Regret}}}

\ifdefined\final
\usepackage[disable]{todonotes}
\else
\usepackage[textsize=tiny]{todonotes}
\fi
\allowdisplaybreaks

\newcommand{\method}{\text{Warm-Start UCB}}

\newcommand{\datashort}{PerMemBench-S\xspace}
\newcommand{\datalong}{PerMemBench-L\xspace}

\newtcolorbox{promptbox}[1]{
    enhanced,
    breakable,
    colback=gray!8,
    colframe=gray!40,
    coltitle=white,
    colbacktitle=gray!55,
    fonttitle=\bfseries,
    title=#1,
    boxrule=0.5pt,
    arc=3pt,
    top=6pt,
    bottom=6pt,
    left=8pt,
    right=8pt,
    titlerule=0pt,
    left*=8pt,
}

\usepackage{soul}

\title{Learning Personalized Memory Gating from Delayed One-Sided Feedback}
\long\def\paperabstract{Long-term memory enables language-model agents to personalize future interactions, but a finite memory bank cannot preserve every extracted fact. Existing memory gating policies are typically shared across users and therefore fail to account for user-specific patterns of future reuse. 
In this paper, we study memory gating at the fact level. We first propose a retrieval-based labeling strategy that labels each fact by whether it is later retrieved and incorporated into responses. We then formulate online memory gating as contextual apple tasting with delayed one-sided feedback, and propose Warm-Start UCB, which combines offline learning from existing data with online adaptation to the current user.
Our theoretical analysis establishes a high-probability regret upper bound of $\tilde{O}(T\epsilon \wedge d\sqrt{T})$ when the delay window is at most $O(\sqrt{T})$, where $T$ is the total number of rounds, $\epsilon$ is the initialization error, $d$ is the dimension of the feature map, and $\tilde{O}(\cdot)$ hides constant and poly-logarithmic factors. The results of our experiment demonstrate that our method outperforms all baselines under different memory budgets and history lengths, establishing it as a promising personalized memory gating policy.}

\author{Kaixuan Ji$^{1,2}$, Zixiang Chen$^{1}$, Wenting Zhao$^{1}$, Cheng Qian$^{1,3}$, Heng Wang$^{1,3}$, Weizhi Zhang$^{1,4}$, Jielin Qiu$^{1}$, Liangwei Yang$^{1}$, Quanquan Gu$^{2}$, Silvio Savarese$^{1}$, Shelby Heinecke$^{1}$, Huan Wang$^{1}$\vspace{1.5mm}\\
$^1$Salesforce AI Research\\
$^2$University of California, Los Angeles\\
$^3$University of Illinois at Urbana-Champaign\\
$^4$University of Illinois Chicago \vspace{1mm}\\
}

\def\month{MM}
\def\year{YYYY}
\def\openreview{\url{https://openreview.net/forum?id=XXXX}}

\begin{document}

\maketitle


\section{Introduction}

Memory systems, which extract and store important information during human-AI interaction, have been widely deployed in large language model (LLM)-based agents~\citep{packer2023memgpt,zhong2024memorybank,chhikara2025mem0}. 
Typically, a memory system extracts candidate facts from each dialogue, stores a subset in a memory bank, and retrieves relevant entries for later queries. Because storage and context budgets are finite, the system must decide which facts are sufficiently valuable to preserve.

Existing systems generally make this decision using criteria shared across users, such as heuristics like whether a fact expresses a stable preference, profile attribute, or ongoing goal~\citep{park2023generative,zhong2024memorybank}, or using policies learned from long-context training data~\citep{yan2026memory,yu2026agentic,sun2026verifiable}. 
Nevertheless, such universal criteria overlook the heterogeneity across different users.
\citet{in2026personalize} formalized this issue at the session level and introduced PerMem-Bench, together with several LLM-based
storage-gating baselines. Their results show substantial headroom between perfect personalized gating and existing causal baselines, motivating methods that learn directly from user-specific interaction outcomes.

To learn directly from user-specific interaction outcomes, we first need to identify a feedback signal grounded in the interaction that reflects the utility of each fact. While a deployed system does not observe whether a discarded fact would have been useful, for a stored fact, later retrieval provides a natural signal of future utility: a fact is utilized if it is hit by the retrieval and incorporated into the response in some later conversation. This implies that \emph{whether a fact gets retrieved and confirmed by the LLM agent can serve as a flag of utility}, and therefore provides natural and reliable guidance for learning a memory gating policy.

The retrieval-based feedback has two outstanding features: the feedback is only available on those added facts, and it becomes meaningful only after some unknown delay. To capture this nature, we formally frame the personalized memory gating problem as an \emph{apple tasting with one-sided delayed feedback} problem~\citep{helmbold1992apple,vernade2020linear}. 
We further propose a two-stage learning strategy, $\method$. The first stage learns the utility pattern from data of existing users. The second stage adapts to user-specific patterns during online interaction via an upper confidence bound (UCB) algorithm.
Our theoretical analysis and empirical results verify the effectiveness of $\method$ as a gating policy. Our contributions are summarized as follows:
\begin{itemize}[leftmargin=*]
    \item We propose a labeling pipeline for learning personalized memory gating, which derives utility labels from downstream retrieval outcomes, refined through LLM-judge verification. This labeling strategy provides a direct feedback for memory gating learning, without relying on heuristic or downstream reward signals.
    \item We frame personalized memory gating as an \emph{apple tasting with one-sided delayed feedback} problem. 
    To solve this problem, we propose a two-stage learning strategy, $\method$.
    Our theoretical analysis implies that $\method$ effectively leverages the offline data while evolving towards user-specific patterns through online interaction.
    \item We empirically evaluate $\method$ on our benchmark across a range of memory budgets and interaction lengths. Experiment results show that $\method$ consistently outperforms previous LLM-based methods and the naive admit-all baseline, suggesting that $\method$ serves as an effective gating policy for personalized memory.
\end{itemize}

\paragraph{Notation.} We use $[n]$ to denote the set $\{1,\dots, n\}$. In this work, we use lowercase letters to represent scalars, and denote vectors and matrices by lower and uppercase boldface letters respectively. Given a vector $\xb\in \RR^d$, we denote the vector's $L_2$-norm by $\|\xb\|_2$. We further define $\|\xb\|_{\bSigma}=\sqrt{\xb^\top\bSigma\xb}$ given a positive semidefinite matrix $\bSigma \in \RR^{d\times d}$.
For two semidefinite matrices $\bSigma, \bLambda \in \RR^{d\times d}$, we use $\bSigma \succeq \bLambda$ if $\xb^\top (\bSigma - \bLambda)\xb \geq 0$ for any $\xb \in \RR^d$. We reserve $\Ib$ for the identity matrix.
For nonnegative sequences $\{x_n\}$ and $\{y_n\}$, we write $x_n = O(y_n)$ interchangeably with $x_n \lesssim y_n$ if $\limsup_{n\to\infty}{x_n}/{y_n} < \infty$, $y_n = \Omega(x_n)$ or $x_n \gtrsim y_n$ if $x_n = O(y_n)$, and $y_n = \Theta(x_n)$, or alternatively, $y_n \sim x_n$, if $x_n = O(y_n)$ and $x_n = \Omega(y_n)$. We further employ $\tilde{O}(\cdot), \tilde{\Omega}(\cdot)$, and $\tilde{\Theta}(\cdot)$ to hide $\polylog$ factors.  We use $\ind\{\cdot\}$ to denote the indicator function. For any $x, y \in \RR$, we denote $x \wedge y = \min\{x,y\}$ and $x \vee y = \max\{x,y\}$.

\section{Related Work}

\paragraph{Personalized Memory System.} Memory systems are widely used in various applications, including embodied systems~\citep{zhu2023ghost,wang2024jarvis}, personal assistants~\citep{wang2025agent,zheng2024synapse} and recommender systems~\citep{wang2024recmind}, incentivizing a long line of work in improving the effectiveness and efficiency. One direction mainly focuses on learning a policy that selectively extracts and persists crucial information from the interaction history~\citep{chhikara2025mem0,zhou2026mem1,yan2026memory,tan2025prospect}. Another direction focuses on leveraging various data structures to capture relation between memory units to improve subsequent retrieval~\citep{xu2026mem,yang2026plugmem,hu2025memory,rezazadeh2025isolated,lu2026choosing}. Despite these advances, most existing systems employ a universal storage policy and therefore fail to fully account for the heterogeneity across different users. \citet{in2026personalize} first formalize this limitation and introduce PerMem-Bench for evaluating personalized memory systems. They further propose several baseline approaches that operate the gating at session level, with LLM-based learning-free methods as the backbone. Compared to~\citet{in2026personalize}, we consider fact-level online adaptation from retrieval-derived feedback, which eliminates the reliance on human or LLM-based heuristic. This results in an apple tasting with one-sided delayed feedback formulation.

\paragraph{Learning with Apple Tasting Feedback.} Apple tasting, introduced by~\citet{helmbold1992apple,helmbold2000apple} is one of the canonical models for one-sided feedback, where the binary outcome is only revealed under the positive action. 
Apple tasting has been studied under several distinct comparators and modeling assumptions. In online binary classification against a hypothesis class, \citet{raman2024apple} characterize agnostic learnability through Littlestone dimension and establish a realizable-case trichotomy; \citet{chase25deterministic} study the additional restriction to deterministic learners.
Apple tasting also appears as a canonical partial-monitoring or feedback-graph problem when regret is measured against a fixed action~\citep{bartok2014partial,alon2015online,cesari2026two}, with rates depending on the precise loss, feedback, and adversary model. 
Contextual variants have been studied through partial monitoring with side information, Bayesian logistic models, feedback-graph reductions, and log-loss~\citep{bartok2012partial,grant2021apple,zhang2023practical,an2026efficient}. Extension has also been made to other settings like fairness-aware~\citep{bechavod2019equal,bechavod2023individually} and strategically manipulated context~\citep{harris2023strategic}.


\paragraph{Learning with Delayed Feedback.} There are long lines of work studying decision making problem with non-immediate feedback, the most canonical among which is that the delay is generated stochastically or adversarially independently of the feedback. In this setting, a regret comparable to the non-delay counterpart up to only additive terms of the delay is provably achievable under both multi-armed setting~\citep{joulani2013online,cesa2019delay,thune2019nonstochastic,bistritz2022no,zimmert2020optimal,gyorgy2021adapting,masoudian2022best,schlisselberg2026improved} and linear setting~\citep{zhou2019learning,howson2023delayed,van2023unified,blanchet2024delay,schlisselberg2026near}. 
Another line of work considers delayed feedback that is entangled with the feedback structure~\citep{lancewicki2021stochastic,tang2024stochastic,schlisselberg2025delay}, among which the most relevant to us are~\citet{vernade2017stochastic,vernade2020linear} that studied the running-zero-flipping-one delay model and achieved $\tilde{O}(d\sqrt{T} + md)$ regret upper bound~\citep{vernade2020linear}, where $d$ is the dimension of feature map and $m$ is a high-probability upper bound of the stochastic delay. In this paper, we adopt the running-zero-flipping-one delay model with apple-tasting feedback to account for the specific structure in agentic memory system, where feedback is only available at added memory entries.

\section{Dataset Construction}

In this section, we introduce a retrieval-based labeling pipeline for learning memory gating policy, and provide a quick analysis for how the utility patterns are shared across users and how they differ. We defer the meta-evaluation of label validity to Appendix~\ref{app:meta-evaluation} and statistics of the constructed benchmark to Appendix~\ref{app:benchmark-statistics}.

\subsection{Benchmark Construction}

Our construction begins with any corpus of user-LLM interaction history
between user $u$ and the LLM, denoted by
$\{\cS_1^{(u)}, \cS_2^{(u)}, \ldots\}$, such that each session is composed of $\cS_t^{(u)} = \{x_1^t, a_1^t, \cdots, x^t_k, a^t_k\}$, where $x_i^t$ is the user's $i$-th prompt and $a_i^t$ the agent's corresponding response in the $t$-th session of the entire interaction history between user $u$ and the agent. In this paper, we consider the corpus generated via the strategy in~\citet{in2026personalize}. Specifically, we adopt the pre-generated PerMem-Bench$^s$ (\datashort), and additionally synthesize a longer dataset (\datalong), to better capture the long-horizon nature of human-AI interaction.


\paragraph{Fact Extraction.}
For each user $u$, we process the sequence of session, $\{\cS_t \}_t$ in chronological order by feeding them into an off-the-shelf memory system\footnote{Here we use \texttt{mem0}~\citep{chhikara2025mem0}, though the pipeline is agnostic to the underlying extractor.}. 
Starting from an empty memory bank $\cM_0 = \varnothing$, for each session $\cS_t$, the memory system 
extracts a collection of candidate atomic facts $\cF_t = \{f_{t,1}, f_{t,2}, \ldots\}$ and then merges these candidate facts into the memory bank $\cM_t = \cM_{t-1} \cup \cF_t$. In plain words, $\cM_t$ is the memory bank before the $t+1$-th session, if we add all extracted facts and the memory bank has infinite budget.


\paragraph{Retrieval-Based Labeling.} After constructing the pool of candidate facts $\{\cM_t\}_t$, we now assign each candidate fact $f_{t,i}$ a binary \emph{utility} label indicating its usefulness.
Recall that a fact can only be useful if it gets retrieved at some later point. Formally, fix a fact $f_{t,i}$, for every user turn $x_{t',i'}$ with $t'>t$, an LLM first decides whether answering that turn requires querying memory at all. For turns judged to require memory, the same LLM formulates a natural-language query $q_{t'}$ and invokes the retriever of the memory system with query $q_{t'}$ against the memory bank state $\cM_{t'-1}$. $f_{t,i}$ is labeled with ``required for $x_{t',i'}$'' only if $f_{t,i}$ appears among the top-$k$ retrieval hits.

\paragraph{LLM-judge Verification.}  A retrieval hit alone is a noisy relevance signal: a fact may be topically related to $q_{t'}$ without actually being needed to answer it (e.g., a temporal mismatch). We therefore filter the obtained retrieval hits with an LLM judge. For each pair of $(x_{t',i'}, f_{t,i})$ such that $f_{t,i}$ hit by the retriever for $x_{t',i'}$, we ask an LLM to judge whether the presence of $f_{t,i}$ in the context is necessary for the answer to $x_{t',i'}$, or substantially improves the answer quality. 
We denote all hits verified by this procedure with
\begin{align*}
    \cH_0 = \{(f_{t,i}, x_{t',i'}): \, f_{t,i} \text{ hit by the retriever for } x_{t',i'}\text{ and verified by the LLM judge}\},
\end{align*}
as all valid retrieval hits. Finally, to mitigate potential neglect of the retrieval system, for all $x_{t',i'}$ where the retrieval results in no strong match (i.e., the scores of all retrieved facts lie below a threshold), we employ an LLM to inspect whether there is a subset of facts from the top retrievals that collectively help with answering the user's prompt. We denote these prompt-fact pairs as $\cH_1 = \{(f_{t,i},x_{t',i'})\}$ and merge them with $\cH_0$ to obtain the final set of valid retrieval hits $\cH = \cH_0 \cup \cH_1$. The final utility label is given by $y_{t,i} = \ind \big\{ \exists \, (t',i'), (f_{t,i}, x_{t', i'}) \in \cH \big\}$.

\paragraph{Evaluation Protocol.} During evaluation, we consider a memory budget with a fixed budget size $B$. For each distinct fact $f$, the gating decides whether it should be added to the memory bank. If the decision is ``add'' and the bank is full, then the earliest entry in the bank will be evicted. Our main metric is retrieval-hit retention rate (RHR), which is the proportion of valid hits in $\cH$ that are realized under the fixed budget. In particular,
\begin{align*}
    & \cH' := \{(f_{t,i}, x_{t',i'}) \in \cH: f_{t, i} \text{ is still in the memory bank at } t'\}, \\
    & \text{RHR} := |\cH'|/|\cH|.
\end{align*}
At a high level, the RHR captures how well a policy preserves facts until the turn where they actually make a difference, given that the memory budget is limited. A miss of fact $f_{t,i}$ at $t'$ stems from one of two reasons: the gating policy drops $f_{t,i}$, or it adds too many non-required facts so that $f_{t,i}$ gets evicted too early. Hence, a policy that attains a high RHR usually adds the required facts to the memory bank and drops unnecessary ones. Besides this, we also consider the F1 score of a gating policy, by treating the storage decision itself as a binary classification problem against the ground truth labels.

\subsection{Meta Analysis}\label{sec:meta-analysis}

In this section, we provide a preliminary study of the heterogeneity and homogeneity among different users. To do this, for each fact, we predict the binary utility label using the TF-IDF representation of the fact content and a logistic regression classifier. We compare two evaluation regimes per user: (i) Self, a stratified $k$-fold cross-validation using only $u$'s own facts, and (ii) leave-one-user-out (LOUO), a single shared-vocabulary model trained on the pooled facts of the other 19 users and evaluated on $u$. 

\begin{table}[h]
\centering
\vspace{0.05in}
\caption{Self (within-user CV) vs.\ LOUO (fixed global model) TF-IDF + logistic
regression AUC for the fact-gating label, averaged over 20 users.}
\label{tab:gating-transfer}
\begin{tabular}{lccc}
\toprule
Dataset & Self AUC & LOUO AUC & Gap  \\
\midrule
\datashort      & 0.756 $\pm$ 0.047 & 0.622 $\pm$ 0.049 & 0.134 \\
\datalong & 0.724 $\pm$ 0.029 & 0.665 $\pm$ 0.024 & 0.059 \\
\bottomrule
\end{tabular}
\end{table}

The results are summarized in Table~\ref{tab:gating-transfer}. First, we observe that the LOUO model itself is far above random (AUC 0.622 and 0.665, respectively, vs. 0.5), meaning that a classifier that has never seen the held-out user's facts still recovers a considerable share of the signal: part of what makes a fact useful for future interactions is shared among users. Despite this, the self AUC exceeds the LOUO AUC prominently, indicating that a substantial, user-specific component of the gating rule does \emph{not} transfer under a single fixed global classifier. This coexistence of homogeneity and heterogeneity indicates that, to make the best use of the data, a policy should both learn from other users and adapt to the current one. This motivates the proposed methods, introduced in the next section.

\section{Algorithm and Analysis}\label{sec:apple-tasting-one-sided-delayed}

In this section, we introduce an online learning model that captures the structure of retrieval-based feedback, along with an algorithm for learning a memory-gating policy from this feedback.

\subsection{Problem Setup}

We model the personalized memory gating problem as a \emph{contextual bandit with apple-tasting, delayed, one-sided feedback} problem and describe the formal setup in this section. For simplicity, we fix a single user and omit the user subscript.


\subsubsection{Interaction Protocol}

At each round $t = 1, \ldots, T$, the learner observes a context $\xb_t \in \RR^d$ and chooses an action $a_t \in \{\actionadd, \actionskip\}$. Here, the context $\xb_t$ is a feature representation\footnote{For example, an embedding of the fact together with its surrounding session context.} of the candidate memory fact extracted at timestamp $t$, and the action $a_t$ is the memory system's gating decision: $\actionadd$ commits the fact to the memory bank, while $\actionskip$ discards it. The learner receives different feedback depending on the action taken. If $a_t =\actionskip$, the learner discards the fact and observes nothing at any future time. If $a_t = \actionadd$, the environment simultaneously samples a delay $D_t > 0$ from a fixed distribution $\cD$, independently, and at each subsequent time $t' > t$, the learner observes the censored conversion indicator
\begin{align*}
z_{t, t'} = y_t \cdot \ind\big\{ D_t \le t' - t \big\},
\end{align*}
where $y_t \in \{0,1\}$ is the actual label such that $y_t=1$ indicates that the fact is utilized by some future user turn, and $y_t=0$ indicates that the fact is never used. Moreover, $D_t$ is the number of rounds until the utilization is confirmed (e.g., the fact is retrieved for a later turn). The observation begins at $z_{t,t}=0$ and, if and only if $y_t=1$, flips to $1$ once the delay elapses. A persistent $0$ is ambiguous between a transient fact and a required fact whose confirming retrieval has not yet occurred. 


\begin{remark}
This one-sided feedback structure, where the outcome of a decision is only ever revealed when the learner chooses the positive action, is classically termed \emph{apple tasting}\footnote{Consider a shopper deciding whether to buy an apple from a bin: they only learn whether it was good or bad by tasting (buying) it, never by leaving it on the shelf.}. Decision-making problems with such a feedback structure have been widely studied in the previous literature~\citep{helmbold1992apple,helmbold2000apple}. Our problem can be seen as a combination of the apple-tasting feedback and the delayed running-zero-flipping-one feedback considered in \citet{vernade2020linear}.
\end{remark}


We assume the learner has some knowledge of the delay distribution.

\begin{assumption}\label{assumption:delay}
    There exist some known $m$ and $\kappa_m = \Omega(1)$ such that $\PP[D_t \leq m] = \kappa_m$, for any $t \in [T]$.
\end{assumption}

Furthermore, we assume $y_t \sim \operatorname{Bernoulli} (\eta(\xb_t))$, where $\eta: \cX \to [0,1]$ is a user-specific function and $\eta(\xb_t)$ is the utility score. A larger $\eta(x)$ means a higher probability that the information in $x$ is used in the future. We assume that there exists some unknown $\btheta^* \in \RR^d$, such that $\eta(\xb_t) = \langle \btheta^*, \xb_t\rangle \in [0,1]$, where the features admit the following regularity conditions.

\begin{assumption}\label{assumption:linear-regularity}
    There exists some known constant $L > 0$ such that $\|\xb_t\|_2 \le L$ for all $t$.
\end{assumption}

\subsubsection{Learning Objective}
With store-gain $b > 0$ and asymmetric error costs $c_{\mathrm{miss}}, c_{\mathrm{store}} > 0$, the reward is
\begin{align*}
r(\xb_t, a_t, y_t) =
\begin{cases}
+b & a_t = \actionadd,\ y_t = 1, \\
-c_{\mathrm{store}} & a_t = \actionadd,\ y_t = 0, \\
-c_{\mathrm{miss}} & a_t = \actionskip,\ y_t = 1, \\
0 & a_t = \mathsf{skip},\ y_t = 0.
\end{cases}
\end{align*}
Without loss of generality, we assume $b + c_{\mathrm{store}} + c_{\mathrm{miss}} = 1$. Taking expectations over $y_t \sim \operatorname{Bernoulli}(\eta(\xb_t))$, the action $\mathsf{add}$ is preferred if and only if $\eta(\xb_t) \ge \tau$, with $\tau = c_{\mathrm{store}}$,
so the oracle is the threshold policy $\pi^\star(\xb) = \mathbf{1}\{ \langle \btheta^\star, \xb \rangle \ge \tau \}$. The learner's goal is to minimize the pseudo-regret against $\pi^\star$,
\begin{align*}
\mathrm{Regret}(T) = \sum_{t=1}^{T} \Big( & \EE_y\big[ r(\xb_t, \pi^\star(\xb_t), y)  -  r(\xb_t, a_t, y) \big] \Big).
\end{align*}

\begin{remark}
While the interaction protocol captures two key features of feedback from retrieval-based labeling, apple-tasting and one-sided delay, we notice that our online learning model still differs from practice in several ways. For example, within all retrieved and verified facts, their utility still varies: some are rarely retrieved and incorporated, whereas others are used more frequently. Also, the cost $c_{\mathrm{store}}$ for adding a non-required fact may vary over time: it will be higher when the memory bank is nearly full and lower when it has more capacity. Extending the model to account for these variations, and developing algorithms for the resulting settings, are promising directions for future work.
\end{remark}

\subsection{Training from Other Profiles}
\label{sec:pretraining}

The meta-analysis in Section~\ref{sec:meta-analysis} suggests that users' storage-preference parameters contain both a shared component and a user-specific component. This motivates using data from other users to construct an informative warm start for the current user. We formalize the quality of this warm start through the following high-probability assumption.

\begin{assumption}\label{assumption:warm-start}
The learner is given a prior center $\bmu_0\in\RR^d$, a positive-definite matrix $\bSigma_0\succeq\Ib$, and a known radius $\epsilon > 0$ such that 
\begin{align*}
    \PP\left(
        \|\btheta^\star-\bmu_0\|_{\bSigma_0}
        \leq \epsilon
    \right)
    \geq 1-\delta_0.
\end{align*}
Here, the probability is taken over both the current user's parameter $\btheta^\star$ and the randomness in the pretraining data used to construct $\bmu_0$ and $\bSigma_0$. 
\end{assumption}

Overall, Assumption~\ref{assumption:warm-start} states that the learner, through techniques like linear regression on other existing users, has access to a sufficiently accurate estimate of the center of the prior distribution of the per-user ground truth $\btheta^*$, which is close to any realized $\btheta^*$ with high probability. In fact, Assumption~\ref{assumption:warm-start} is plausible given the meta-analysis in Section~\ref{sec:meta-analysis} and standard analyses of random-design linear regression~\citep{hsu2014random}. In particular, let $\overline{\btheta}$ denote a population-level parameter shared across users. The transferability observed in Section~\ref{sec:meta-analysis} suggests that the current user's parameter is concentrated around this common component. Therefore, it is reasonable to make the assumption that, with probability at least $1-\delta_{\mathrm{pop}}$,
\begin{align}
    \|\btheta^\star-\overline{\btheta}\|_{\bSigma_0}
    \leq \epsilon_{\mathrm{pop}}.
    \label{eq:population-concentration}
\end{align}
At the same time, the cross-user predictability demonstrated in Section~\ref{sec:meta-analysis} suggests that data from other users can be used to estimate the shared component. Standard random-design linear-regression results, such as \citet[Theorem~2]{hsu2014random}, can be used to establish that, with probability at least $1-\delta_{\mathrm{est}}$,
\begin{align}
    \|\bmu_0-\overline{\btheta}\|_{\bSigma_0}
    \leq \epsilon_{\mathrm{est}}.
    \label{eq:population-estimation}
\end{align}
Combining \eqref{eq:population-concentration} and \eqref{eq:population-estimation} with the triangle inequality gives $\|\btheta^\star-\bmu_0\|_{\bSigma_0} \leq \epsilon_{\mathrm{pop}}+\epsilon_{\mathrm{est}}$ with probability at least $1-\delta_{\mathrm{pop}}-\delta_{\mathrm{est}}$. Thus, Assumption~\ref{assumption:warm-start} holds with $\epsilon = \epsilon_{\mathrm{pop}}+\epsilon_{\mathrm{est}}$ and $\delta_0 = \delta_{\mathrm{pop}}+\delta_{\mathrm{est}}$. We refer to Appendix~\ref{app:discuss-assumption-pretrain} for a detailed discussion.

\subsection{Online Adaptation}

\begin{algorithm}[t]
    \caption{Warm-Started Ridge UCB with One-Sided Delayed Feedback ($\method$)}
    \label{alg:delayed-linear-prior}
    \begin{algorithmic}[1]
    \REQUIRE Threshold $\tau$, warm-start center $\bmu_0$, prior shape
    $\bSigma_0\succeq\Ib$, regularization parameter $\lambda\geq L^2$,
    confidence level $\delta$, delay parameters $(m,\kappa_m)$, horizon $T$
    \STATE Initialize $\bV_0\leftarrow\lambda\bSigma_0$, $\bb_0\leftarrow \lambda\bSigma_0(\kappa_m\bmu_0)$, $\widehat{\btheta}_0 \leftarrow\kappa_m\bmu_0$, $\alpha_0\leftarrow\beta_{\lambda,\delta}$.
    \FOR{$t=1,2,\ldots,T$}
        \STATE Observe context $\xb_t\in\RR^d$
        \STATE Compute $\widetilde\eta_t \leftarrow \langle\widehat{\btheta}_{t-1},\xb_t\rangle + \alpha_{t-1} \|\xb_t\|_{\bV_{t-1}^{-1}}$
        \STATE Play $a_t\leftarrow$ $\actionadd$ if $\widetilde\eta_t\geq\kappa_m\tau$. Otherwise play $a_t\leftarrow$ $\actionskip$
        \STATE Observe all delayed feedback available at the end of round $t$
        \STATE Set $\Vb_t \leftarrow \lambda\bSigma_0 + \sum_{s\leq t} \ind\{a_s=\actionadd\}\xb_s\xb_s^\top$, $\bb_t \leftarrow \lambda\bSigma_0(\kappa_m\bmu_0) + \sum_{s\leq t} \ind\{a_s=\actionadd\}\xb_s\bar z_{s,t}$, and $\widehat{\btheta}_t \leftarrow \Vb_t^{-1}\bb_t$
        \STATE Set $\alpha_t \leftarrow \beta_{\lambda,\delta} + \sum_{s\in\mathcal P_t} \|\xb_s\|_{\Vb_t^{-1}}$
    \ENDFOR
    \end{algorithmic}
\end{algorithm}

While a parameter pre-trained on other profiles might already serve as a good proxy for most individual profiles, it does not account for the heterogeneity revealed in Section~\ref{sec:meta-analysis}, and therefore leads to linear regret as the time horizon $T \to \infty$. To address this issue, we introduce an online adaptation strategy that starts from the pre-trained parameter $\bmu_0$ but still adjusts towards the user's specific parameter during the online interaction. The strategy is summarized in Algorithm~\ref{alg:delayed-linear-prior}. Similar warm-start strategies have also been applied in previous bandit papers such as \citet{cella2020meta,peleg2022metalearning}.

Algorithm~\ref{alg:delayed-linear-prior} extends previous algorithms for one-sided delayed feedback~\citep{vernade2020linear} to the apple-tasting feedback setting. At each round, Algorithm~\ref{alg:delayed-linear-prior} observes the context $\xb_t$ and computes an index $\tilde \eta_t$ that is optimistic for $\kappa_m\eta(\xb_t)$ with high probability, inspired by previous work on linear contextual bandits~\citep{abbasi2011improved}. Then the algorithm plays $\actionadd$ if and only if the index exceeds the threshold $\kappa_m \tau$. Finally, the algorithm updates its parameter estimate based on the most recently available feedback.

The key challenge arises from the estimation of the parameter. For every added sample $s$ and current time step $t \geq s$, a running-zero label is ambiguous: it might be either that the fact is not useful (i.e., $y_s=0$) or that it is still under the delay (i.e., $D_s > t-s$). To address this issue, we follow the approach in~\citet{vernade2020linear} and truncate the observation by a window of length $m$: for each added sample $s$, we wait for $m$ rounds and treat its label as $1$ if we observe a flip, and $0$ otherwise, \emph{even if it flips to $1$ later}. Formally, define the truncated observation
\begin{align*}
    \bar z_{s,t} :=  y_s\ind \big\{  D_s\leq \min\{t-s,m\} \big\}.
\end{align*}
Our next step is to treat the truncated observation $\bar z_{s,t}$ as the regression label. Notice that
\begin{align*}
    \EE[\bar z_{s,s+m}\mid\xb_s] &= \PP[D_s \leq m] \eta(\xb_s) = \langle\kappa_m\btheta^\star,\xb_s\rangle.
\end{align*}
Consequently, the regression on the feature-label pairs $\{(\xb_s, \bar z_{s,t})\}_{s=1}^{t-m}$ targets $\kappa_m\btheta^\star$. Since our regression target is $\kappa_m\btheta^\star$ rather than $\btheta^\star$, we compensate for the scaling by multiplying our estimate $\tilde \eta_t$ by $\kappa_m^{-1}$ when comparing it to the threshold $\tau$. Based on the regression, we compute the confidence radius of the estimate
\begin{align*}
    \beta_{\lambda,\delta} = \sigma \sqrt{d\ell_\lambda+2\log(1/\delta)} + \kappa_m\sqrt{\lambda}\epsilon,
\end{align*}
where $\ell_\lambda := \log (1+TL^2/(\lambda d))$, and $\sigma$ is a sub-Gaussian parameter of the noise, which can be taken to be $1/2$ in our setting.
Finally, because observations from the most recent $m$ added samples may not yet have reached their final truncated values, the confidence radius $\alpha_t$ contains an additional term corresponding to the pending-observation rounds $\cP_t = \{ s\leq t: a_s=\actionadd,t-s<m \}$
\begin{align*}
\alpha_t = \beta_{\lambda,\delta} + \sum_{s\in\cP_t}\|\xb_s\|_{\Vb_t^{-1}}.
\end{align*}

\begin{remark}
Unlike the original algorithm of \citet{vernade2020linear}, our decision rule compares the estimated utility $\eta(\xb)$ against a fixed, exogenous threshold $\tau$. In contrast, comparisons among contemporaneous arms are invariant to a common multiplicative factor. Consequently, our algorithm must know $\kappa_m=\PP(D\leq m)$, or have a sufficiently accurate estimate of it, in order to rescale the decision threshold.
\end{remark}

\begin{table*}[ht]
\centering
\caption{Results on the \datashort benchmark. We use \textbf{bold} for the highest score and \underline{underline} for the second highest. For LLM-based methods, ``(Ours)'' indicates the corresponding method with our fine-grained prompt. 
The results show that $\method$ outperforms the Universal baseline and LLM-based methods uniformly across all budget sizes.
}
\vspace{0.05in}
\begin{tabular}{l | cccc|cccc}
\toprule
Metric        & \multicolumn{4}{c|}{Micro Average} & \multicolumn{4}{c}{Macro Average} \\
Budget        & 60     & 120    & 250    & 500    & 60     & 120    & 250    & 500    \\
\midrule
Oracle        & 69.34  & 92.99  & 100.00    & 100.00    & 77.28  & 95.22  & 100.00    & 100.00    \\
\midrule
Universal     & 17.44  & 35.94  & \underline{56.93}  & 80.60   & 19.34  & 39.29  & 60.76  & \underline{84.91}  \\
Greedy        & 23.01  & 32.95  & 44.49  & 47.34  & 24.86  & 34.16  & 45.08  & 47.56  \\
Greedy (Ours)      & 23.94  & 35.99  & 53.70   & 68.13  & 26.11  & 38.29  & 55.00     & 67.87  \\
Context   & 23.39 & 32.28 & 42.96 & 45.28 & 25.33 & 33.20  & 42.95 & 45.21 \\
Context (Ours) & 23.34 & 36.37 & 52.11 & 66.32 & 24.72 & 37.75 & 53.09 & 66.29     \\
\midrule
\rowcolor{LightCyan} $\methodcoldstart$  & 20.55 & 37.63 & 56.62 & \underline{82.13} & 23.29 & 41.49 & \underline{60.95} & 83.93 \\
\rowcolor{LightCyan} $\methodpretrained$   & \textbf{24.35} & \underline{38.62} & 56.90  & 79.64 & \textbf{26.52} & \underline{41.72} & 59.73 & 81.06 \\
\rowcolor{LightCyan} $\method$ & \underline{24.10}  & \textbf{38.64} & \textbf{58.73} & \textbf{84.48} & \underline{26.40}  & \textbf{41.98} & \textbf{61.84} & \textbf{86.94} \\
\bottomrule
\end{tabular}
\label{tab:main-results-short}
\end{table*}

\begin{theorem}\label{thm:main}
Suppose Assumptions~\ref{assumption:delay},~\ref{assumption:linear-regularity} and~\ref{assumption:warm-start} hold. Fix any $\delta\in(0,1)$. Then, with a proper choice of $\lambda$ that depends on $T,d, m, \kappa_m, \epsilon$, for sufficiently small $\epsilon$, with probability at least $1-\delta_0-\delta$,  Algorithm~\ref{alg:delayed-linear-prior} satisfies
\begin{align*}
    \regret(T)\leq\widetilde O\bigg(\min\bigg\{T\epsilon,\frac{d\sqrt T+md}{\kappa_m}\bigg\}\bigg),
\end{align*}
where $\tilde O(\cdot)$ hides all logarithmic factors and the dependency on $L$.
\end{theorem}

Theorem~\ref{thm:main} shows that, when the pre-trained estimate $\bmu_0$ is accurate enough, the per-step regret is roughly bounded by $\epsilon$, so the algorithm only incurs $T\epsilon$ regret when $T$ is small. On the other hand, when $T$ is large, the algorithm enjoys a sublinear $\tilde{O}(d\sqrt{T})$ regret, indicating that the proposed strategy effectively learns the user-specific parameter.

\section{Experiments}

\begin{table*}[ht]
\centering
\caption{Results on the \datalong benchmark. We use \textbf{bold} for the highest score and \underline{underline} for the second highest. For LLM-based methods, ``(Ours)'' indicates the corresponding method with our fine-grained prompt. 
The results show that $\method$ significantly outperforms the Universal baseline and LLM-based methods uniformly across all budget sizes.
}
\vspace{0.05in}
\begin{tabular}{l|ccccc|ccccc}
\toprule
Metric & \multicolumn{5}{c|}{Micro Average} & \multicolumn{5}{c}{Macro Average} \\
Budget Size  & 120  & 250  & 500  & 800  & 1000  & 120  & 250  & 500 & 800  & 1000  \\
\midrule
Oracle        & 54.30 & 78.61 & 97.51 & 100.00 & 100.00
              & 58.68 & 82.62 & 98.33 & 100.00 & 100.00 \\
\midrule
Universal     & 20.17 & 35.28 & 50.77 & 65.73 & 73.77
              & 22.03 & 37.13 & 52.55 & 67.50 & 75.71 \\
Greedy        & 19.65 & 29.19 & 40.59 & 47.42 & 49.19
              & 20.38 & 30.02 & 40.53 & 46.50 & 47.68 \\
Greedy (Ours) & 18.64 & 29.87 & 45.55 & 58.12 & 62.87
              & 19.54 & 31.32 & 46.59 & 58.56 & 63.10 \\
Context       & 17.01 & 25.43 & 36.30 & 41.56 & 43.49
              & 17.70 & 25.93 & 36.50 & 40.61 & 42.52 \\
Context (Ours)& 18.76 & 28.87 & 45.13 & 57.33 & 63.49
              & 19.79 & 30.13 & 46.76 & 58.54 & 64.40 \\
\midrule
\rowcolor{LightCyan}
$\methodcoldstart$           & 20.75 & 35.46 & 51.22 & \underline{66.82} & 74.42
              & 22.12 & 37.33 & 53.09 & \underline{68.67} & \underline{76.43} \\
\rowcolor{LightCyan}
$\methodpretrained$   & \underline{24.72} & \textbf{38.08} & \underline{55.21} & 66.13 & \underline{75.02}
              & \underline{26.44} & \textbf{39.98} & \underline{56.88} & 67.02 & 76.05 \\
\rowcolor{LightCyan}
$\method$          & \textbf{25.22} & \underline{37.77} & \textbf{56.14} & \textbf{69.21} & \textbf{76.46}
              & \textbf{26.73} & \underline{39.87} & \textbf{57.92} & \textbf{69.93} & \textbf{77.52} \\
\bottomrule
\end{tabular}
\label{tab:main-results-long}
\end{table*}

\subsection{Baselines and Methods}


\paragraph{Model-Free Baselines.} We follow~\citet{in2026personalize} and consider the following two baselines that do not operate on a specific model
\begin{itemize}[leftmargin=*]
    \item \textbf{Oracle:} The memory system is provided with the ground truth labels and only adds the required facts, serving as the skyline of the performances.
    \item \textbf{Universal:} The memory system adds every extracted fact into the memory bank, representing current memory systems.
\end{itemize}

\paragraph{LLM-Based Methods.} We also consider the causal LLM-based baselines introduced in~\citet{in2026personalize}. For both LLM-based baselines, we use \texttt{gpt-5-mini} as the backbone and evaluate two prompting strategies. One follows the prompt of~\citet{in2026personalize}, while the other, denoted by ``Ours'', uses a more fine-grained rubric that additionally accounts for the remaining capacity of the memory bank. The complete prompts are provided in Appendix~\ref{app:prompts-llm-based}.
\begin{itemize}[leftmargin=*]
    \item \textbf{LLM-Greedy.} For each extracted fact, an LLM predicts whether the fact captures crucial information that is likely to be useful in a future session. The prediction is based only on the current fact, without access to the user's previous interaction history. 
    \item \textbf{LLM-Context-Aware.} After each session, the dialogue is summarized in one or two sentences. For each subsequent fact, the LLM is provided with the current fact, the current session, and a sliding window containing the most recent $K$ session summaries, and is asked to determine whether the fact, given the context, might be required in the future.
\end{itemize}


\paragraph{Linear Gating Methods.} We consider three linear gating methods, all of which use the same linear feature representation of the facts. Besides the $\method$ introduced in Algorithm~\ref{alg:delayed-linear-prior}, we consider its two variants as its ablations. We refer to Appendix~\ref{app:hyper-parameter-ucb} for the implementation details.
\begin{itemize}[leftmargin=*]
    \item \textbf{$\methodcoldstart$.} Algorithm~\ref{alg:delayed-linear-prior} run with $\bmu_0=\mathbf{0}$ and no access to other users' data. The gating rule is learned using only the current user's feedback, which isolates the contribution of online adaptation.
    \item \textbf{$\methodpretrained$.} A linear scorer $\bmu_0$ fitted offline on the labeled facts of other users and kept frozen throughout the interaction, without adaptation to the current user's feedback, which isolates the contribution of cross-user pre-training.
\end{itemize}

\begin{figure*}[t]
\centering   
\subfigure[120]{\label{fig:long-120}\includegraphics[width=0.32\textwidth]{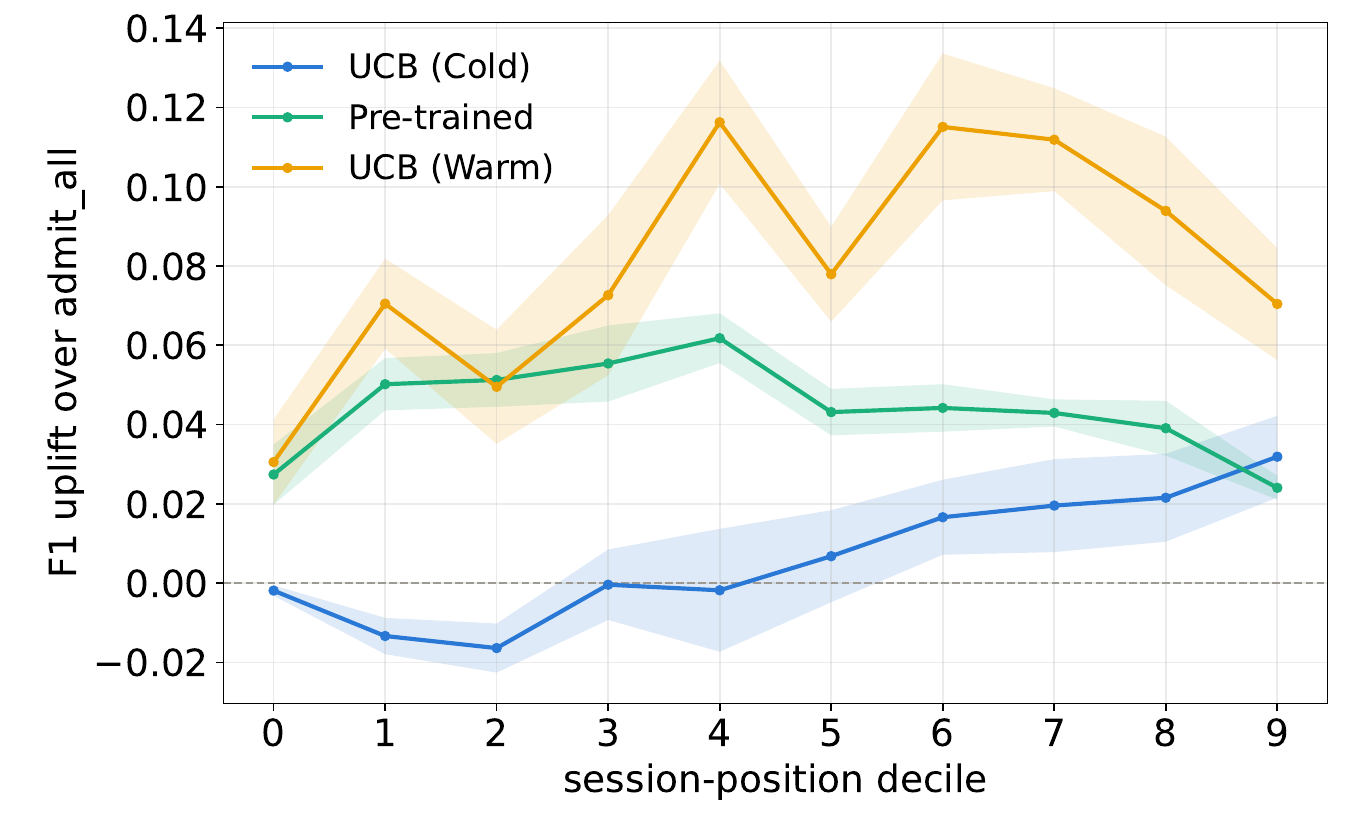}}
\subfigure[500]{\label{fig:long-500}\includegraphics[width=0.32\textwidth]{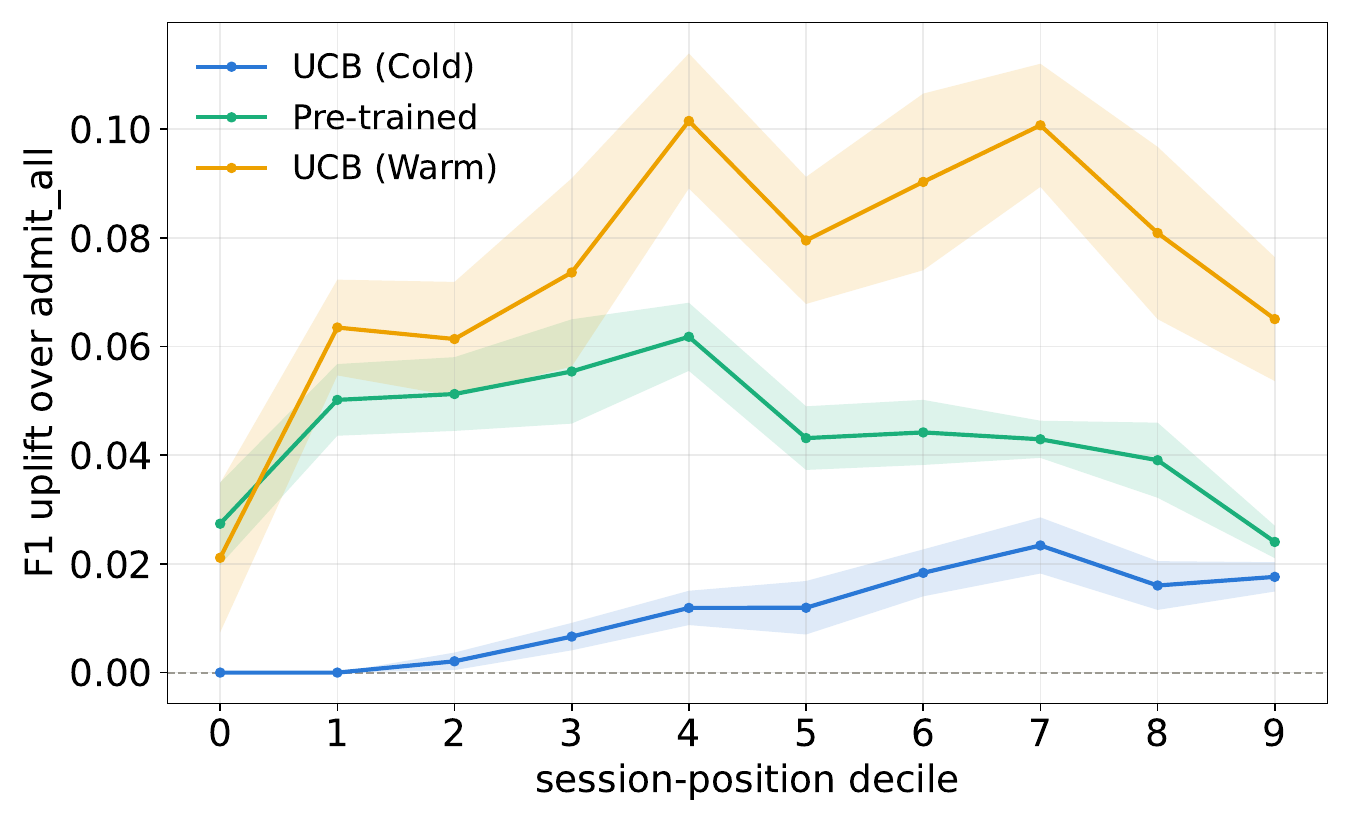}}
\subfigure[1000]{\label{fig:long-1000}\includegraphics[width=0.32\textwidth]{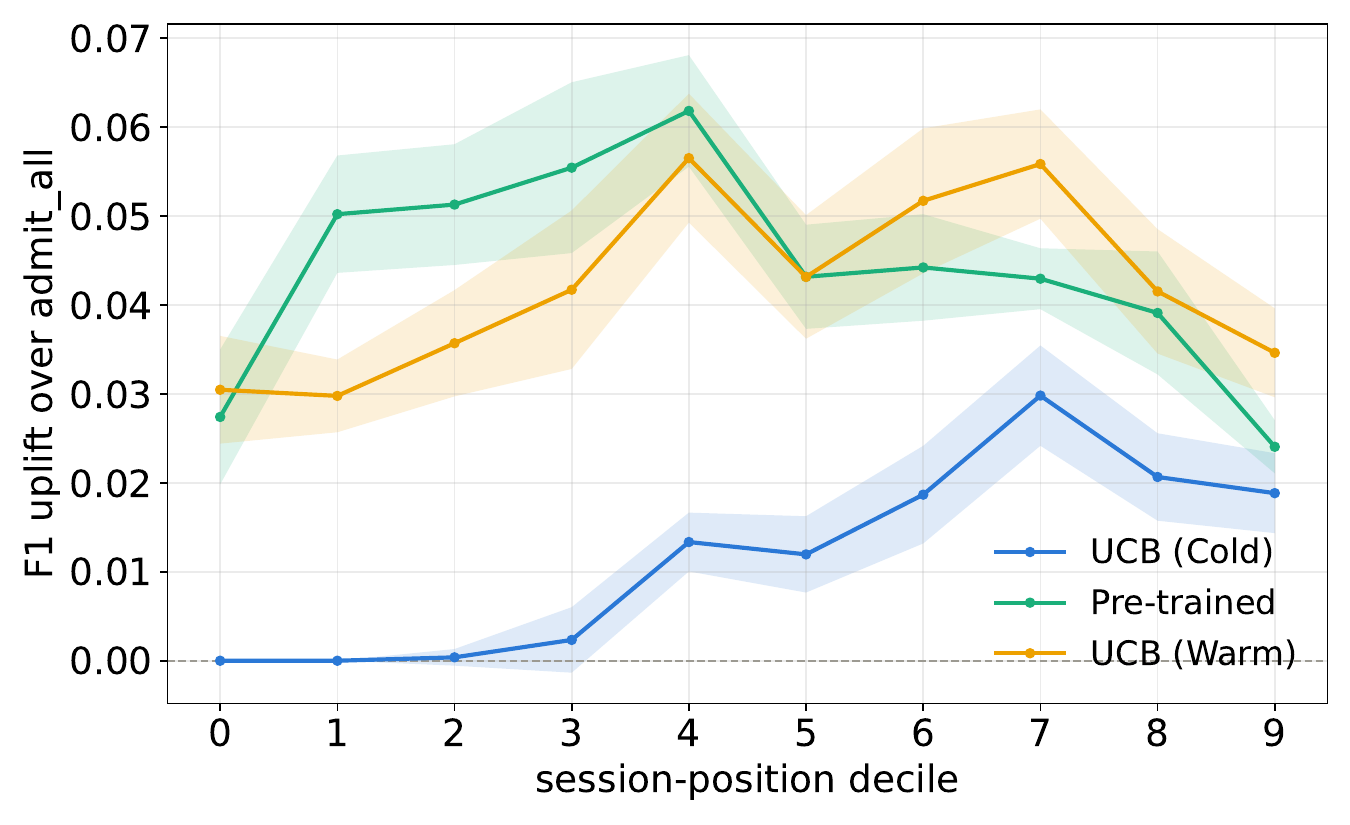}}
\vspace{-0.05in}
\caption{Learning dynamics on \datalong: F1 gain over the Universal baseline in each session-position decile under memory budgets $B \in \{120, 500, 1000\}$. Shaded regions denote the 95\% CI over all users. The gating quality improves in later deciles for online methods $\methodcoldstart$ and $\method$.}
\label{fig:learning-curve-long}
\end{figure*}

\begin{figure*}[t]
\centering   
\subfigure[120]{\label{fig:long-pt-120}\includegraphics[width=0.32\textwidth]{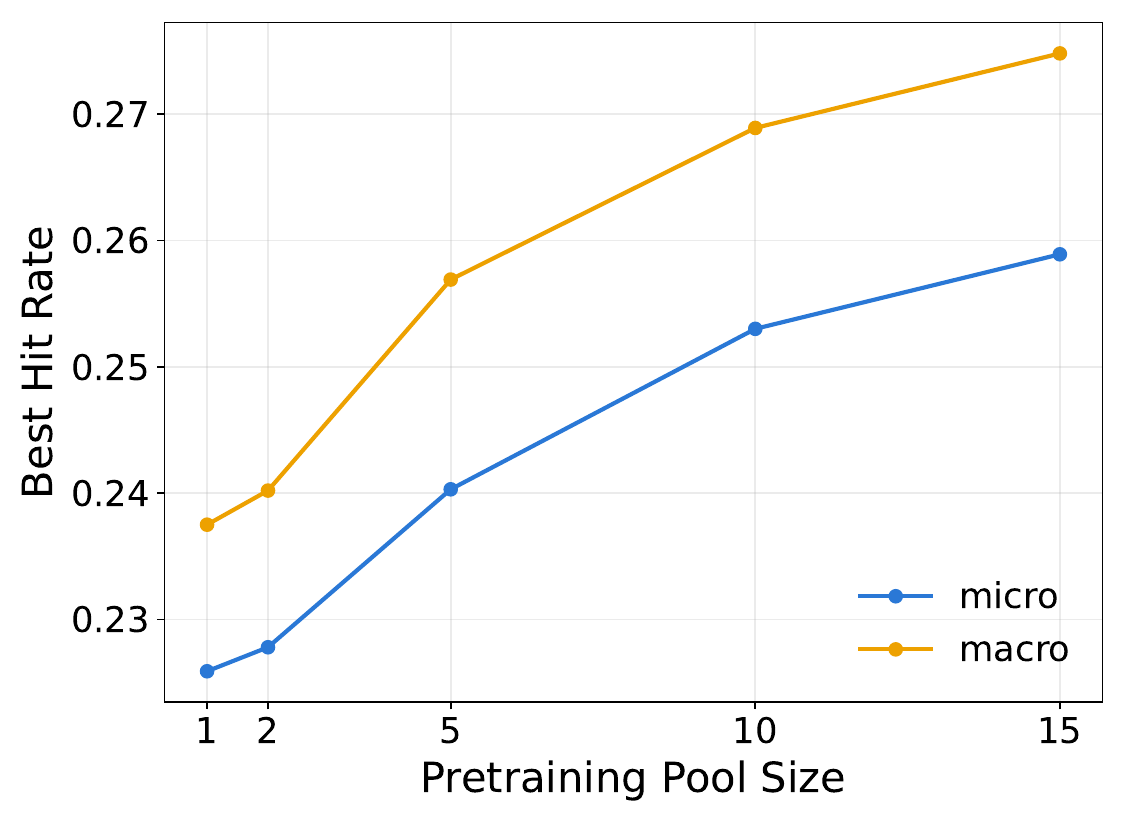}}
\subfigure[500]{\label{fig:long-pt-500}\includegraphics[width=0.32\textwidth]{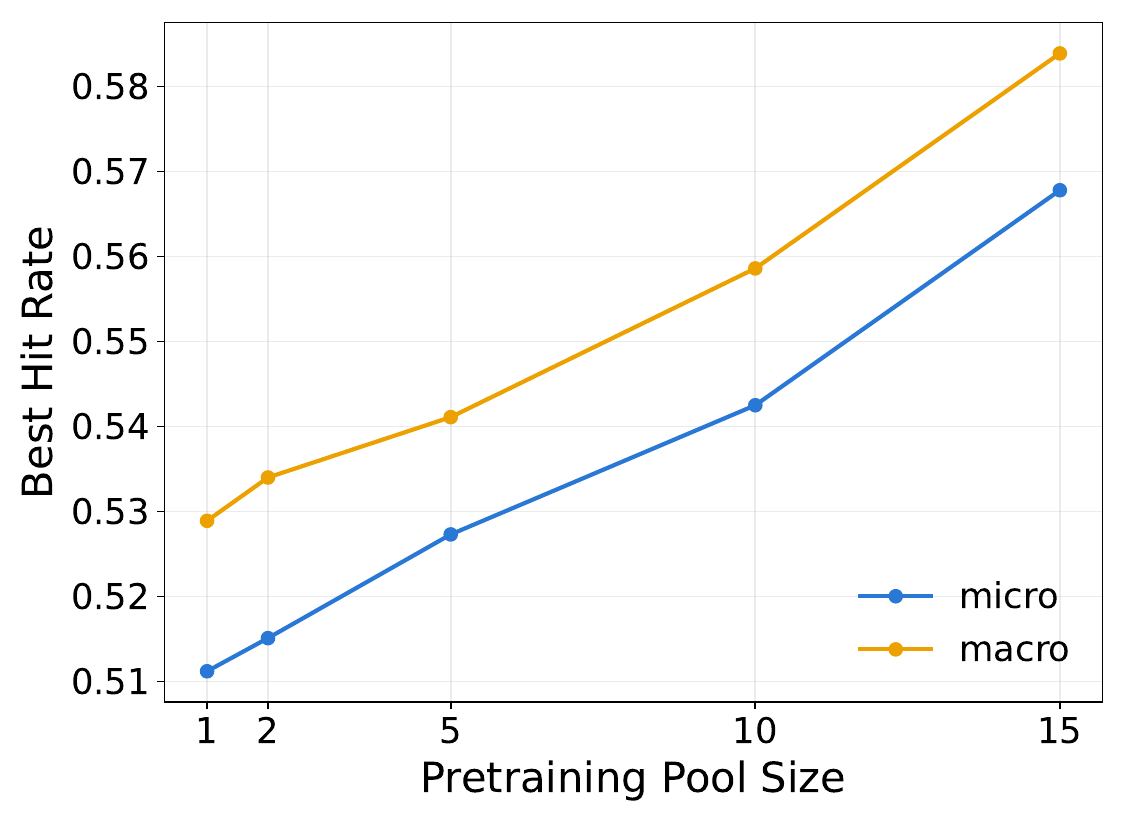}}
\subfigure[1000]{\label{fig:long-pt-1000}\includegraphics[width=0.32\textwidth]{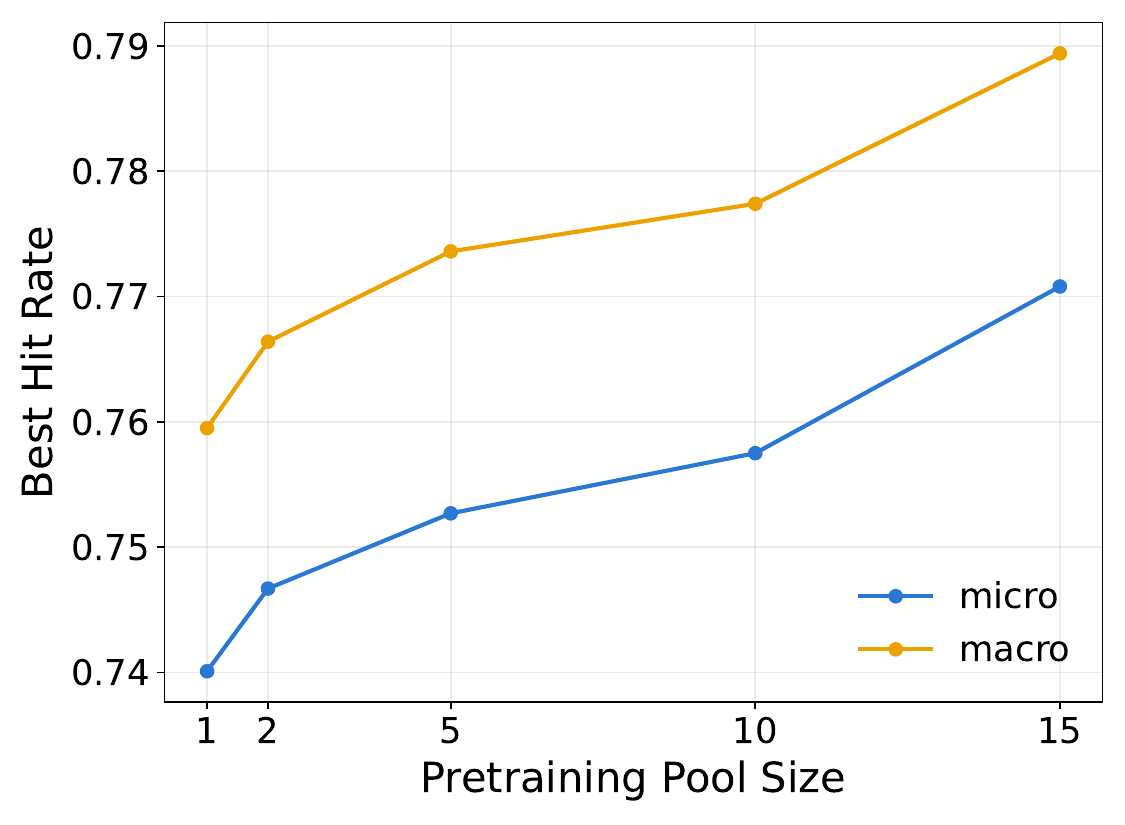}}
\vspace{-0.05in}
\caption{Effect of the pre-training pool size on \datalong: best micro- and macro-averaged RHRs under $B \in \{120, 500, 1000\}$. The results show that the performance continues to improve as the pretraining pool size grows.}
\vspace{-0.05in}
\label{fig:pretrain-size-curve-long}
\end{figure*}

\subsection{Experiment Results}

The results are summarized in Table~\ref{tab:main-results-short} and Table~\ref{tab:main-results-long}. We report the RHR as a percentage under both micro and macro averaging. Micro averaging pools valid retrieval hits across users, whereas macro averaging first computes the RHR separately for each user and then averages across users.

\noindent\textbf{Results on \datashort.}
As shown in Table~\ref{tab:main-results-short}, $\method$ consistently outperforms the Universal baseline and all LLM-based gating methods across every memory budget under both micro and macro averaging. For example, at $B=60$, our method performs on par with the best LLM-based methods and improves over the Universal baseline by $6.66$ micro and $7.06$ macro percentage points. On the other hand, when the memory budget is large, LLM-based methods underperform the Universal baseline, while our method still improves over it by $1.80$ micro and $1.08$ macro points at $B=250$, and by $3.88$ micro and $2.03$ macro points at $B=500$. To summarize, the results indicate that our method consistently outperforms the considered baselines.

Meanwhile, the comparison among the three linear methods further reveals the complementary roles of pre-training and online adaptation. Under the smallest budget, the feedback signal is sparse, so the improvement mostly comes from $\methodpretrained$. On the other hand, with a larger budget, more online feedback becomes available, which improves the performance of UCB. Since our proposed method combines the two strategies, it yields a stable improvement across all budget sizes.


\noindent\textbf{Results on \datalong.}
Table~\ref{tab:main-results-long} shows a consistent pattern on the substantially longer interaction trajectories of \datalong. Our method, $\method$, achieves the best non-oracle result in eight of the ten evaluation columns and the second-best result in the remaining two, where it is narrowly below $\methodpretrained$ at $B=250$. Compared with the Universal baseline, our method improves the micro RHR by $2.49$--$5.37$ percentage points and the macro RHR by $1.81$--$5.37$ points across the five evaluated budgets. 

Moreover, a longer interaction history means that both pre-training and online adaptation are exposed to more data and feedback. Nevertheless, the results show that, compared to the Universal baseline, $\methodpretrained$ leads to a considerable improvement, while UCB without initialization performs only marginally above the Universal baseline. Despite this, we still observe a consistent improvement brought by online adaptation when the memory budget is large. This result indicates that a population-level gating rule transfers useful information to a new user, but cannot fully capture the user's individual storage preferences. Updating the rule with user-specific feedback is therefore essential for maintaining performance over long interaction horizons.


\begin{remark}
Across all experimental settings, we see that LLM-based methods are comparable to the Universal baseline under small budgets. However, they become inferior to the Universal baseline under larger budgets, and the gap increases with the memory budget size. This observation is consistent with the results in~\citet{in2026personalize}. We also observe that a more fine-grained, capacity-aware prompt only mitigates this inferiority but hardly resolves it. 
\end{remark}

\subsection{Ablation Studies}\label{sec:ablation}

\noindent\textbf{Learning Dynamics.} To examine whether the online methods genuinely adapt to the current user, we partition each user's interaction trajectory into ten session-position deciles and report the F1-score improvement over the Universal baseline within each decile. Figure~\ref{fig:learning-curve-long} presents the results on \datalong for representative memory budgets; the remaining budgets and the corresponding \datashort results are provided in Appendix~\ref{app:additional-experiment}.

As suggested by the plots, $\methodcoldstart$ begins with little or even negative improvement over the Universal baseline because it has not yet received enough one-sided, delayed feedback to learn the current user's gating rule. Its performance gradually improves in later deciles, confirming that the online updates extract useful information from the user's interaction history. In contrast, $\methodpretrained$ provides an immediate positive uplift but exhibits comparatively limited improvement as the interaction proceeds because its parameter remains fixed. $\method$ combines the strengths of both methods. It begins with the transferable knowledge supplied by pre-training and then increasingly separates from the frozen policy after user-specific feedback becomes available, also verifying the conceptual superiority of $\method$.


\noindent\textbf{Effect of the Pre-training Pool Size.} We next vary the number of user profiles included in the pre-training pool while keeping the held-out test user fixed. Figure~\ref{fig:pretrain-size-curve-long} reports the best micro and macro RHRs on \datalong for pre-training pool sizes in $\{1,2,5,10,15\}$. Performance improves monotonically with the number of pre-training profiles under every displayed budget and both averaging schemes. Notably, the curves continue to improve at a pool size of 15, especially for $B=500$ and $B=1000$, suggesting that the pre-trained scorer has not yet saturated and may benefit from a larger and more diverse collection of user profiles.



\section{Conclusion and Future Work}

In this paper, we study the personalized memory gating problem. We first provide a retrieval-based utility labeling pipeline, based on which we construct a dataset for memory-gating evaluation. We further frame the problem as an apple tasting with one-sided delayed feedback problem, and propose an algorithm for this problem that achieves an $\tilde{O}\big(T\epsilon \wedge (d\sqrt{T}+md)\big)$ regret guarantee, where $m$ is the delay window,, as demonstrated by our theoretical analysis. Experiments on our constructed dataset show that our method consistently outperforms all baselines, confirming its effectiveness. Currently, our method assumes a stable parameter for each user. Accounting for the evolution of the utility pattern remains an interesting direction for future work.

\bibliography{ref}
\bibliographystyle{tmlr}


\clearpage
\appendix

\section{Benchmark Construction Details}

\subsection{Prompts in Dataset Construction}

In this section, we provide the detailed prompts used in the construction of the dataset. For the stages up to and including dialogue generation, we use the same prompts as~\citet{in2026personalize} and refer readers to that paper for details.

\begin{promptbox}{Retrieval Trigger Prompt}
\scriptsize
You are deciding, on behalf of a memory-augmented assistant, whether responding well to a user's message requires recalling something from an earlier conversation with this same user.

User's message: 
\begin{verbatim}
{user_content}
\end{verbatim}

First, decide: does responding well to this message require recalling a specific fact, preference, plan, or detail the user has shared before -- as opposed to being answerable from this message alone or from general knowledge?

If a memory lookup is needed, compose a short, self-contained search query capturing specifically what needs to be recalled -- not a copy of the whole message, just the retrieval-relevant part.

Return strict JSON only, no other text:
\begin{verbatim}
{{"query_needed": true or false, "query": "<query text, or empty string if not needed>"}}
\end{verbatim}
\end{promptbox}

\begin{promptbox}{Retrieval Verification Prompt}
\scriptsize
You are judging whether retrieved memory facts are potentially helpful to respond to a user's message in an ongoing conversation.

User message:
\begin{verbatim}
{query_text}
\end{verbatim}

Candidate fact:
\begin{verbatim}
{candidate}
\end{verbatim}

For each candidate, decide whether it is potentially helpful to respond well to this message. A fact is necessary if omitting it would make the response worse in any of these ways:
- Factually wrong or incomplete given what the user has told you
- Generic or one-size-fits-all where a tailored response is possible
- Inconsistent with a preference, constraint, or decision the user has already established
- Missing personalization the user would reasonably expect, given they already shared this

A fact is NOT necessary if it is merely on the same topic or in the same domain as the message, but doesn't actually change what a good response would say -- e.g. it's superseded, tangential, or the message doesn't call for that level of detail.

Being topically related is not sufficient on its own, but genuine personalization value (not just factual correctness) does count as necessary. If a candidate is a close call, prefer to keep it -- this check does not need to be the final word.

Return strict JSON only, no other text, with the result and your reasoning:
\begin{verbatim}
{{"helpful": <yes|no>, "reasoning": "<your reasoning>"}} 
\end{verbatim}
(empty list if none are necessary)
\end{promptbox}

\begin{promptbox}{Retrieval Recall Prompt}
\scriptsize
You are judging retrieved memory facts against a user's message in an ongoing conversation. None of these facts individually scored as a strong match, but they may still matter in combination.

User message:
\begin{verbatim}
{query_text}
\end{verbatim}

Candidate fact:
\begin{verbatim}
{format_candidates(candidates)}
\end{verbatim}

Is there any subset of these facts that, used together, would be useful or necessary to respond to the user's message? A single fact by itself not being a strong match does not rule out a combination of two or more being relevant.

Return strict JSON only, no other text:
\begin{verbatim}
{{"useful_entry_ids": ["<id>", ...]}}
\end{verbatim}
(empty list if no combination is useful)
\end{promptbox}

\begin{promptbox}{Quality Assessment Prompt}
\scriptsize
You are checking whether a specific piece of background information was actually used in an assistant's reply within an ongoing conversation.

Conversation so far:
\begin{verbatim}
{session_prefix}
\end{verbatim}

Background fact under test:
\begin{verbatim}
{fact_content}
\end{verbatim}

Assistant's reply:
\begin{verbatim}
{response}
\end{verbatim}

Classify how much the reply draws on that specific fact:

- contains info from fact: the reply includes specific content, details, or claims that come from this fact (not just a generic reply that happens to be compatible with it).

- weakly related: the reply touches the same general topic as the fact but doesn't actually use its specific content.

- irrelevant: the reply shows no sign of drawing on this fact at all.

Return strict JSON only, no other text:
\begin{verbatim}
{{"incorporation": "<contains_info_from_fact|weakly_related|irrelevant>"}}
\end{verbatim}

\end{promptbox}

\subsection{Meta Evaluation}\label{app:meta-evaluation}

In this section, we conduct experiments to verify that the labels produced by the introduced pipeline genuinely reflect the actual utility during later interaction, rather than merely being an artifact of the selected memory system and LLM-judge. 

\paragraph{Method.} We sample collection of probes (i.e., user's prompt) with valid retrieval hits. For each sample, we sample one judge-confirmed fact and one unconfirmed candidate, either not retrieved or rejected by the LLM-judge. For each sampled fact, we place it in the context of an LLM (we used \texttt{gpt-5-mini} here) and draw $k$ completions at non-zero temperature. For each generated response, we employ \texttt{gpt-5-mini} as a separate judge model to decide if the generated response related to the candidate fact: contains the fact, weakly related or irrelevant, and mapping to a score within $\{1, 0.5, 0\}$ and average over the $k$ responses. 
For the incorporation level test, we conduct a Wilcoxon signed-rank test over all pairs of (probe, average score of confirmed fact, average score of unconfirmed fact), and report the $p$-value, effect size $r$, and the 95\% confidence interval of the scores.

\begin{table}[ht]
  \centering
  \caption{Per-user fact-incorporation scores (mean, with 95\% bootstrap
    CI) for confirmed vs.\ unconfirmed target facts, and the paired
    Wilcoxon test comparing them within each user. Reported as a
    robustness check on the pooled result, not as independent
    significance tests.}
  \label{tab:meta-per-uuids}
  \begin{tabular}{lrcccc}
    \toprule
    User & $n$ & Confirmed & Unconfirmed & $r$ & $p$ \\
    \midrule
    0045d22a & 12 & 0.600~[0.42, 0.78] & 0.167~[0.03, 0.33] & 1.00 & $1.95\times10^{-3}$ \\
    0097794e & 29 & 0.510~[0.36, 0.66] & 0.155~[0.07, 0.26] & 0.75 & $1.04\times10^{-3}$ \\
    009850d8 & 18 & 0.372~[0.20, 0.55] & 0.144~[0.02, 0.31] & 1.00 & $4.98\times10^{-3}$ \\
    00aefb8e & 32 & 0.456~[0.32, 0.60] & 0.091~[0.03, 0.18] & 0.90 & $1.11\times10^{-4}$ \\
    00e85c7f & 22 & 0.482~[0.35, 0.62] & 0.077~[0.01, 0.15] & 0.95 & $4.07\times10^{-4}$ \\
    \midrule
    \textbf{Pooled} & \textbf{113} & \textbf{0.477~[0.40, 0.55]} & \textbf{0.121~[0.08, 0.17]} & \textbf{0.88} & \textbf{$6.95\times10^{-13}$} \\
    \bottomrule
  \end{tabular}
\end{table}

\paragraph{Results.} We sampled $5$ users and pooled all
$n{=}113$ matched confirmed and unconfirmed pairs across the five evaluated users, of which the results are presented in Table~\ref{tab:meta-per-uuids}. The pooled results show that, compared to unconfirmed facts, confirmed facts are significantly more incorporated into the responses. Moreover, a paired Wilcoxon signed-rank test on the matched per-probe scores rejects the null of no difference ($W{=}220.5$, $p \approx 7\times 10^{-13}$) with a large matched-pairs rank-biserial effect size ($r{=}0.88$), indicating a statistical significance. 
Moreover, when breaking down to the per-user results, we find that all of these results agree with the pooled results. For every user, confirmed facts are incorporated into responses much more often than unconfirmed facts. Mean scores range from $0.372$ to $0.600$ for confirmed facts and from $0.077$ to $0.167$ for unconfirmed facts, leading to a per-user gap of $0.23$ to $0.43$. Moreover, the effect size is at least $0.75$ and $p < 0.005$ for all users, indicating that the difference stays significant across five users, showing that the pooled result is not driven by any single user: facts with positive labels are consistently used in the responses.

\subsection{Statistics of the Constructed Benchmark}\label{app:benchmark-statistics}

Tables~\ref{tab:dataset-statistics-short} and~\ref{tab:dataset-statistics-long} summarize some basic statistics of the constructed dataset. We highlight several notable properties that shape the gating problem.

\begin{table}[ht]
\centering
\vspace*{0.05in}
\caption{Per-user statistics of \datashort, computed over the 20 users. Delay-to-first-hit is measured in sessions from extraction to the first valid retrieval hit.}
\vspace{0.05in}
\begin{tabular}{l | ccccc}
\toprule
Metric                                 & Mean    & Median  & Min    & Max     & Std    \\
\midrule
Sessions per uuid                      & 54.25   & 57      & 36     & 70      & 9.93   \\
Total entries per uuid                 & 856.5   & 894     & 474    & 1295    & 203.58 \\
Entries with valid hits per uuid & 124.5   & 122     & 45     & 241     & 56.93  \\
Hit rate (hit entries / total entries) & 0.142 & 0.140 & 0.064 & 0.236 & 0.047 \\
Avg delay-to-first-hit (sessions)      & 13.47   & 13.02   & 6.95   & 23.2    & 3.76  \\
\bottomrule
\end{tabular}
\label{tab:dataset-statistics-short}
\end{table}

\begin{table}[ht]
\centering
\vspace*{0.05in}
\caption{Per-user statistics of \datalong, computed over the 20 users. Delay-to-first-hit is measured in sessions from extraction to the first valid retrieval hit.}
\vspace{0.05in}
\begin{tabular}{l | ccccc}
\toprule
Metric                                 & Mean    & Median  & Min    & Max     & Std    \\
\midrule
Sessions per uuid                      & 150.35  & 153    & 94    & 193   & 30.84  \\
Total entries per uuid                 & 2449.65 & 2435   & 1378  & 3341  & 550.39 \\
Entries with valid hits per uuid       & 493.9   & 518    & 225   & 780   & 170.45 \\
Hit rate (hit entries / total entries) & 0.199   & 0.203  & 0.113 & 0.256 & 0.043  \\
Avg delay-to-first-hit (sessions)      & 30.44   & 28.64  & 19.38 & 41.08 & 5.86  \\
\bottomrule
\end{tabular}
\label{tab:dataset-statistics-long}
\end{table}

 
\paragraph{Sparsity of Utilized Facts.}
On average, only $14.2\%$ of a user's facts are utilized in \datashort and $19.9\%$ in \datalong. Even the user with the highest ratio reaches only $23.6\%$ and $25.6\%$, respectively. In other words, for typical user, more than three quarters of the extracted facts are never used to answer a later query. Consequently, a policy that stores every fact fills a limited memory bank mostly with facts that are never used, and exhausts the bank capacity quickly. This leaves substantial room for a gating policy that learns to skip un-needed facts.
 
\paragraph{Relatively Long Delay of Feedback.}
On average, a utilized fact receives its first valid hit $13.47$ sessions after extraction in \datashort and $30.44$ sessions in \datalong. These delays are roughly a quarter and a fifth of the average history length, respectively, and they grow with the length of the history rather than vanishing. With about $16$ facts extracted per session on both benchmarks, this corresponds to several hundred subsequent facts, which exceeds some of the memory budgets we evaluate. This indicates that, feedback on an added fact arrives only after a substantial fraction of the interaction has passed, and might even get evicted before a meaningful signal is observed. Thus, the delay presents as an essential source of hardness for learning and motivates both the delay-aware estimator in Algorithm~\ref{alg:delayed-linear-prior} and the warm start from other users' data.


\section{Experiment Details}

\subsection{Prompts of LLM-Based Methods}\label{app:prompts-llm-based}

In this section, we provide the details of the prompts we used for LLM-greedy gating and LLM-context-aware gating. For each method, we consider two prompt templates: one directly adapted from~\citet{in2026personalize}, and a fine-grained version that takes the remaining capacity of the memory bank into consideration.

\paragraph{Greedy.} At each fact, the LLM receives the content of the current fact and outputs a binary decision.

\begin{promptbox}{LLM-Greedy Gating Prompt~\citep{in2026personalize}}
\scriptsize
You are a memory policy agent. Decide whether this fact should be stored in long-term memory.

\#\# Fact
\begin{verbatim}
{fact.get("content", "")}
\end{verbatim}

store = true: durable, worth-storing information about the user -- part of an ongoing project, recurring goal, preference, or constraint.
store = false: transient fact -- trivial, one-off, or self-contained with no lasting relevance.

Respond ONLY with: 
\begin{verbatim}
{{"store": <true|false>}}
\end{verbatim}
\end{promptbox}

\begin{promptbox}{LLM-Greedy Gating Prompt with Remaining Capacity}
\scriptsize
You are the admission gate for a long-term memory bank in a memory-augmented assistant. The bank has a hard capacity limit -- every fact you admit competes with every other fact for a slot.

\begin{verbatim}
{capacity_block(history)}

{GATING_RUBRIC}
\end{verbatim}

Fact:
\begin{verbatim}
{fact.get("content", "")}
\end{verbatim}

First briefly note which KEEP or DROP category the fact best matches, then decide.

Return strict JSON only, no other text:
\begin{verbatim}
{{"category": "<short KEEP/DROP category label>", "store": true or false}}
\end{verbatim}
\end{promptbox}

\paragraph{Context-Aware.} Before prediction, the current session is summarized in 1–2 sentences and appended to a running summary buffer. At each session, the LLM receives the current dialogue alongside
the most recent K session summaries as historical context.

\begin{promptbox}{LLM-Context-Aware Gating Prompt~\citep{in2026personalize}}
\scriptsize
You are a memory policy agent. Decide whether this fact should be stored in long-term memory.

\#\# Recent Session Summaries 
\begin{verbatim}
{history_summaries}
\end{verbatim}

\#\# Current Session Dialogue
\begin{verbatim}
{current_dialogue}
\end{verbatim}

\#\# Fact
\begin{verbatim}
{fact.get("content", "")}
\end{verbatim}

store = true: durable, worth-storing information about the user -- part of an ongoing project, recurring goal, preference, or constraint.
store = false: transient fact -- trivial, one-off, or self-contained with no lasting relevance.

Respond ONLY with: 
\begin{verbatim}
{{"store": <true|false>}}
\end{verbatim}
\end{promptbox}

\begin{promptbox}{LLM-Context-Aware Gating Prompt with Remaining Capacity}
\scriptsize
You are the admission gate for a long-term memory bank in a memory-augmented assistant. The bank has a hard capacity limit -- every fact you admit competes with every other fact for a slot.

\begin{verbatim}
{capacity_block(history)}

{GATING_RUBRIC}
\end{verbatim}

\#\# Recent Session Summaries 
\begin{verbatim}
{history_summaries}
\end{verbatim}

\#\# Current Session Dialogue
\begin{verbatim}
{current_dialogue}
\end{verbatim}

Fact:
\begin{verbatim}
{fact.get("content", "")}
\end{verbatim}

First briefly note which KEEP or DROP category the fact best matches, then decide.

Return strict JSON only, no other text:
\begin{verbatim}
{{"category": "<short KEEP/DROP category label>", "store": true or false}}
\end{verbatim}
\end{promptbox}

The contents of the \textit{gating rubric} and \textit{capacity block} are given as follows.

\begin{promptbox}{Gating Rubric}
\scriptsize
KEEP (durable, user-specific -- worth recalling in a future conversation):

- Concrete preferences ("prefers X over Y", likes/dislikes)

- Standing plans, goals, or decisions the user has committed to

- Constraints or limitations (budget, health, schedule, skill level, access)

- Relationships, roles, or people/pets recurring in the user's life

- Stable biographical or contextual facts (job, location, ongoing projects)

DROP (low value -- not worth the slot):

- Small talk, pleasantries, filler, or meta-conversation

- One-off transient state ("is tired right now", "it's raining today")

- Generic statements that aren't specific to this user

- Speculative or hypothetical musings the user hasn't actually decided on

- Facts that are near-duplicates of something already obviously established
\end{promptbox}

\begin{promptbox}{Capacity Block}
\scriptsize

\begin{verbatim}
If vacant
\end{verbatim}

Capacity check: the memory bank has plenty of free room right now, so bias toward KEEPING. Admit anything that plausibly fits a KEEP category below, even if it's only moderately durable or specific -- only DROP facts that clearly belong in the DROP list.

\begin{verbatim}
If somewhat vacant
\end{verbatim}

Capacity check: the memory bank is filling up, so be selective. Admit only facts that clearly fit a KEEP category below with a reasonable chance of mattering later -- DROP anything vague, borderline, or that fits the DROP list.

\begin{verbatim}
If space limited
\end{verbatim}

Capacity check: the memory bank is nearly full, so be highly selective. Admit only facts you are confident are high-value KEEPs (firm decisions, explicit constraints or preferences the user is likely to reference again, major ongoing plans) -- DROP anything minor, speculative, or merely plausible.
\end{promptbox}

\subsection{Hyper-parameters of Proposed Methods}\label{app:hyper-parameter-ucb}

In this section, we provide the implementation details and hyper-parameters of linear gating methods. We use OpenAI \texttt{text-embedding-3-small} to obtain an embedding of each fact and treat the embedding as the linear feature. For $\method$ and $\methodpretrained$, for each given test user, we first fit a linear parameter on a training set composed of $15$ other users randomly selected from the pool of $19$ users. We then select the hyper-parameters by validating on another $4$ users. We keep the routine of Algorithm~\ref{alg:delayed-linear-prior} unchanged, while making the confidence radii $\alpha_t$ and the gating threshold $\tau$ as hyper-parameters. We further simplify the initialization of the pre-trained covariance matrix $\bSigma_0$ by setting it to identity matrix $\Ib$, for $\method$.

For $\methodcoldstart$, we consider the confidence radius $\alpha  \in \{0.1, \, 0.2,\, 0.5,\, 1.0\}$ and the gating threshold $\tau \in \{0.05,\, 0.1,\, 0.2,\, 0.5 \}$ and $\lambda \in \{0.1, 1, 10\}$. For each fact, we wait for $m$ sessions to see if there is at least one valid retrieval. We consider $m \in \{8,10,15\}$ for \datashort and $m \in \{25, 35, 45\}$ for \datalong, to account for the longer delay. We report the best performance achieved by the combination of hyper-parameters that achieves the highest micro RHR. 

For $\methodpretrained$, we sweep the gating threshold $\tau \in \{0.05,\, 0.1,\, 0.2,\, 0.5 \}$. For UCB with warm start, we adapt the gating threshold by considering $\tau \in \{0.5\tau_{\text{pt}}, \tau_{\text{pt}} ,2\tau_{\text{pt}}\}$, where $\tau_{\text{pt}}$ is the best validated value for $\methodpretrained$. We further consider the confidence radius $\alpha  \in \{0.1, \, 0.2,\, 0.5,\, 1.0\}$. For \datalong, we consider $ \lambda  \in \{ 200, \, 100, \, 50, \, 20 \}$ and $m  \in \{25, \, 35\}$. For \datashort, we account for the smaller amount of pre-training data and the shorter delay, and therefore consider $\lambda \in \{ 20,\, 10,\, 5,\, 2 \}$ and $m  \in \{10, \, 15\}$.

\section{Additional Experiment Results}\label{app:additional-experiment}

This section reports the learning-dynamics and pre-training pool-size results omitted from Section~\ref{sec:ablation}. Figures~\ref{fig:learning-curve-short} and~\ref{fig:pretrain-size-curve-short-app} show the per-decile F1 gain over the Universal baseline and the effect of the pre-training pool size on \datashort, respectively. The same set of results for the remaining budgets of \datalong (i.e., $250$ and $800$) is presented in Figure~\ref{fig:learning-curve-long-app}.

\begin{figure*}[ht]
\centering   
\subfigure[250]{\label{fig:long-250}\includegraphics[width=0.26\textwidth]{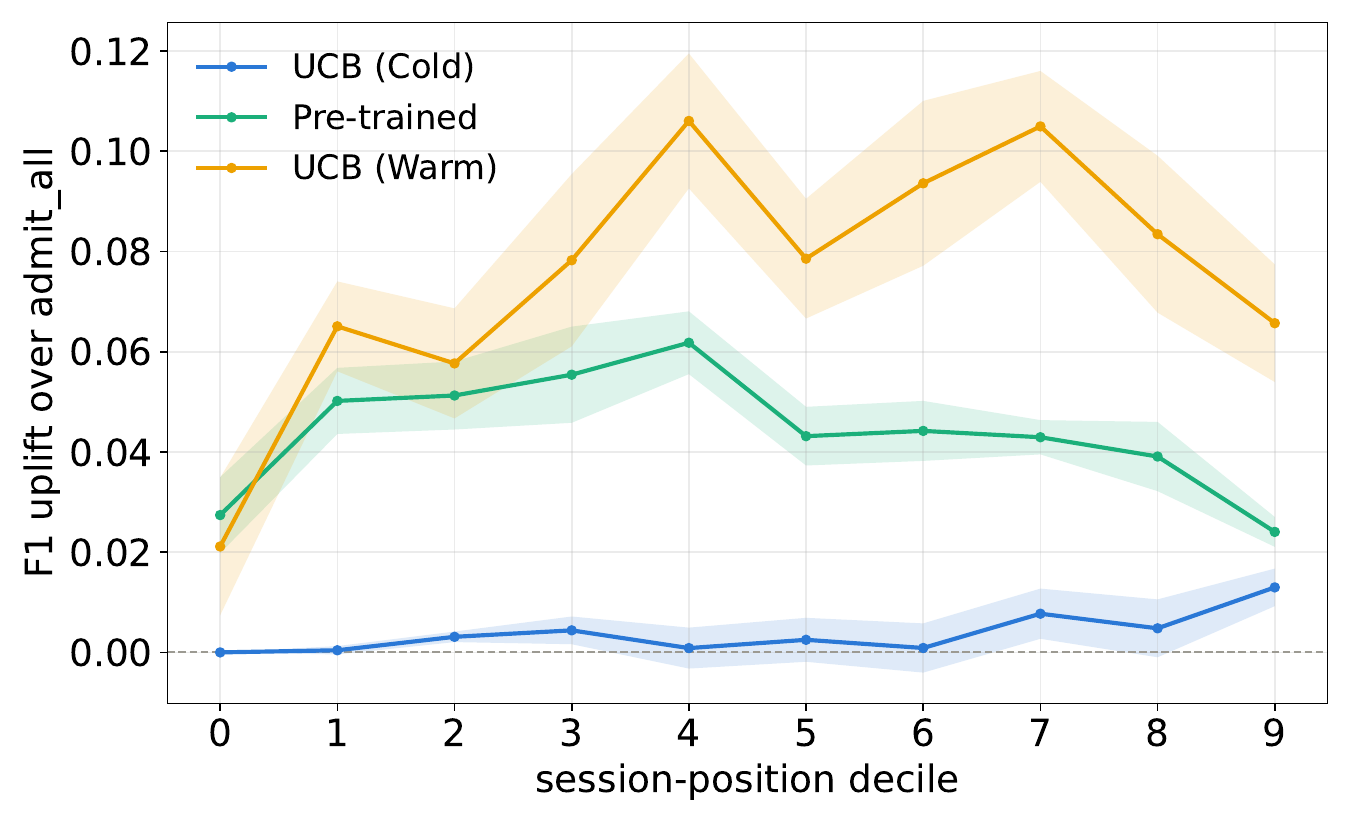}}
\subfigure[800]{\label{fig:long-800}\includegraphics[width=0.26\textwidth]{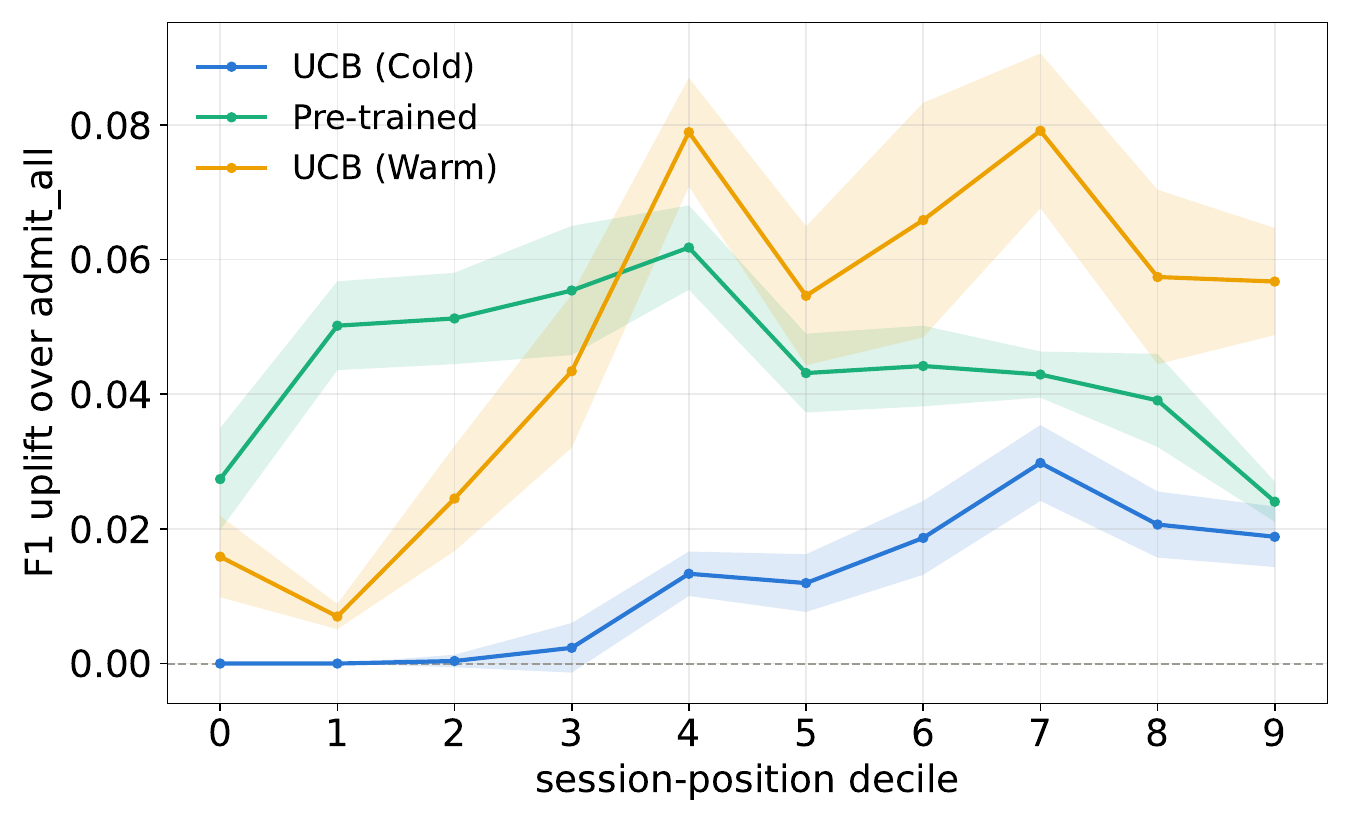}}
\subfigure[250]{\label{fig:long-pt-250}\includegraphics[width=0.22\textwidth]{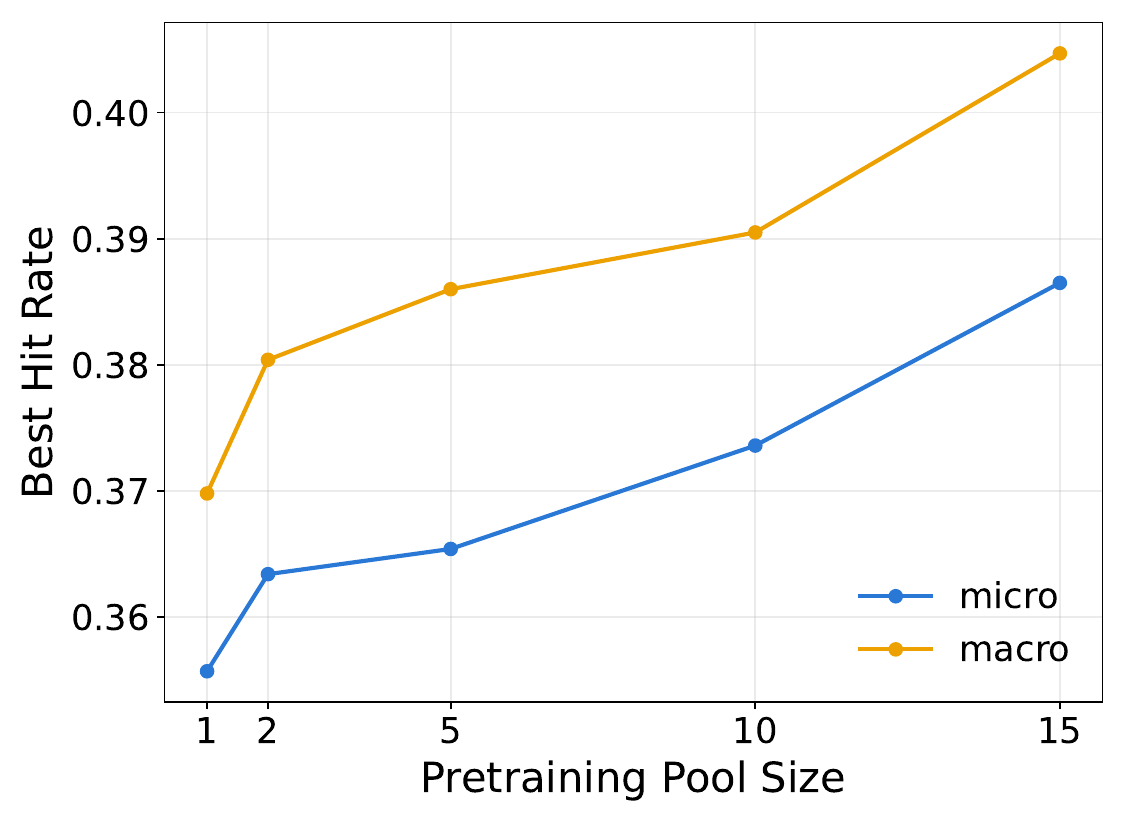}}
\subfigure[800]{\label{fig:long-pt-800}\includegraphics[width=0.22\textwidth]{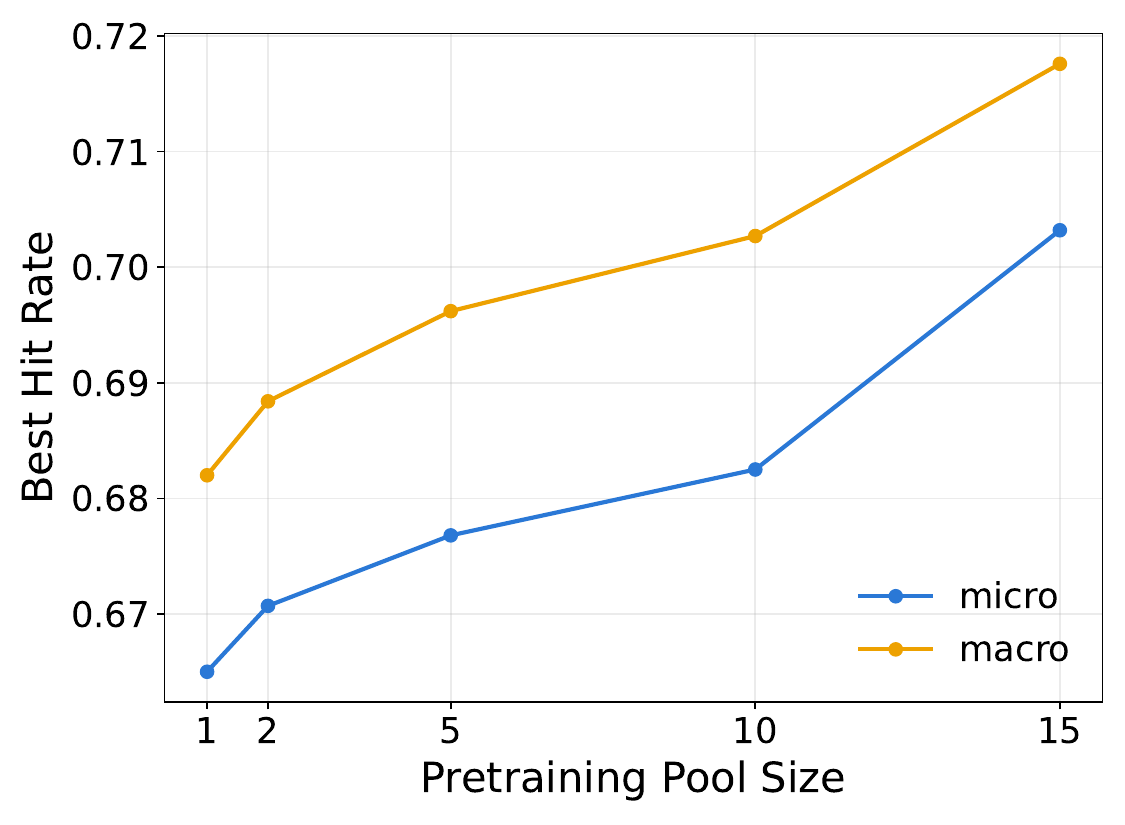}}
\caption{Learning dynamics~(\ref{fig:long-250}, \ref{fig:long-800}) and effect of the pre-training pool size~(\ref{fig:long-pt-250}, \ref{fig:long-pt-800}) on \datalong for the remaining memory budgets $B \in \{250, 800\}$. Shaded regions denote the 95\% CI over all users.}
\label{fig:learning-curve-long-app}
\end{figure*}


\begin{figure*}[ht]
\centering
\subfigure[60]{\label{fig:short-60}\includegraphics[width=0.24\textwidth]{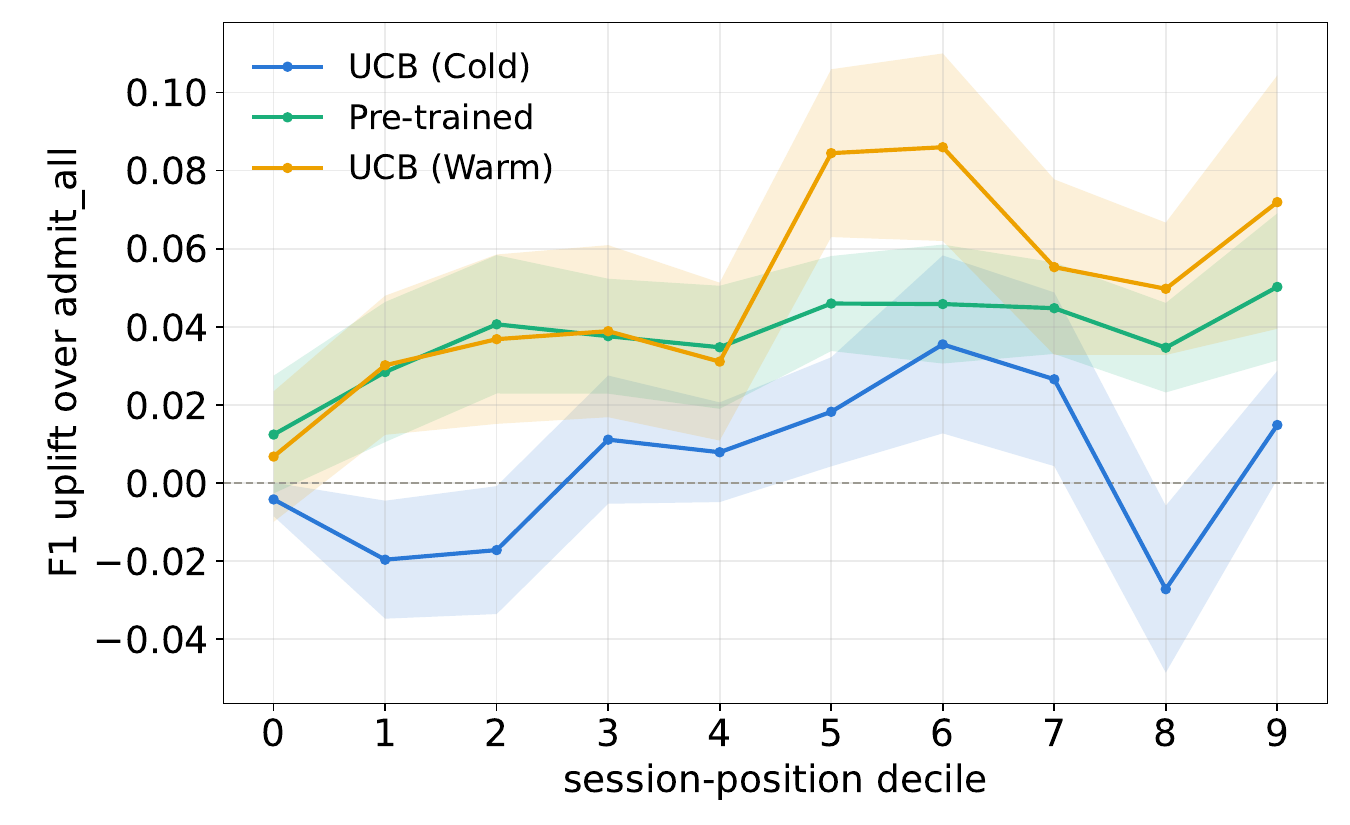}}
\subfigure[120]{\label{fig:short-120}\includegraphics[width=0.24\textwidth]{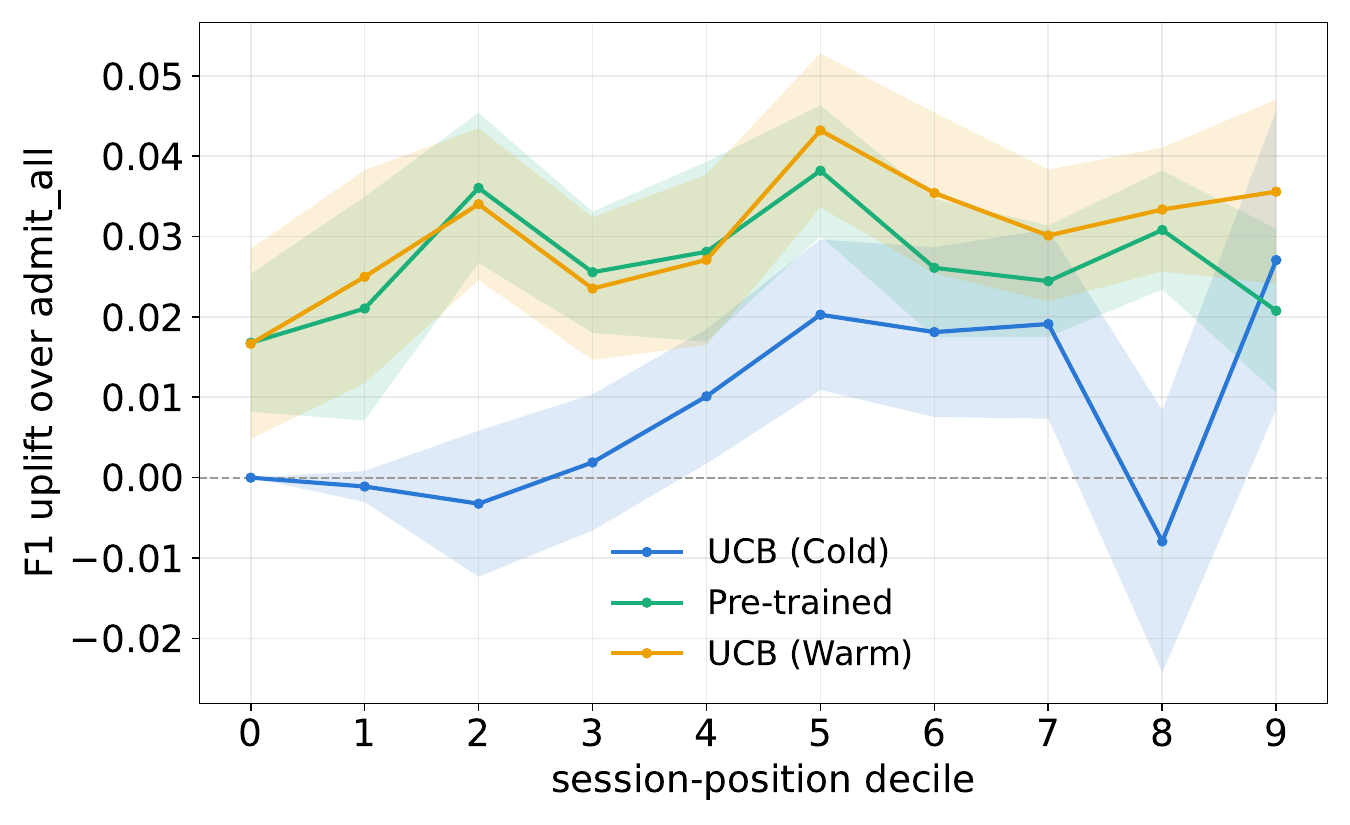}}
\subfigure[250]{\label{fig:short-250}\includegraphics[width=0.24\textwidth]{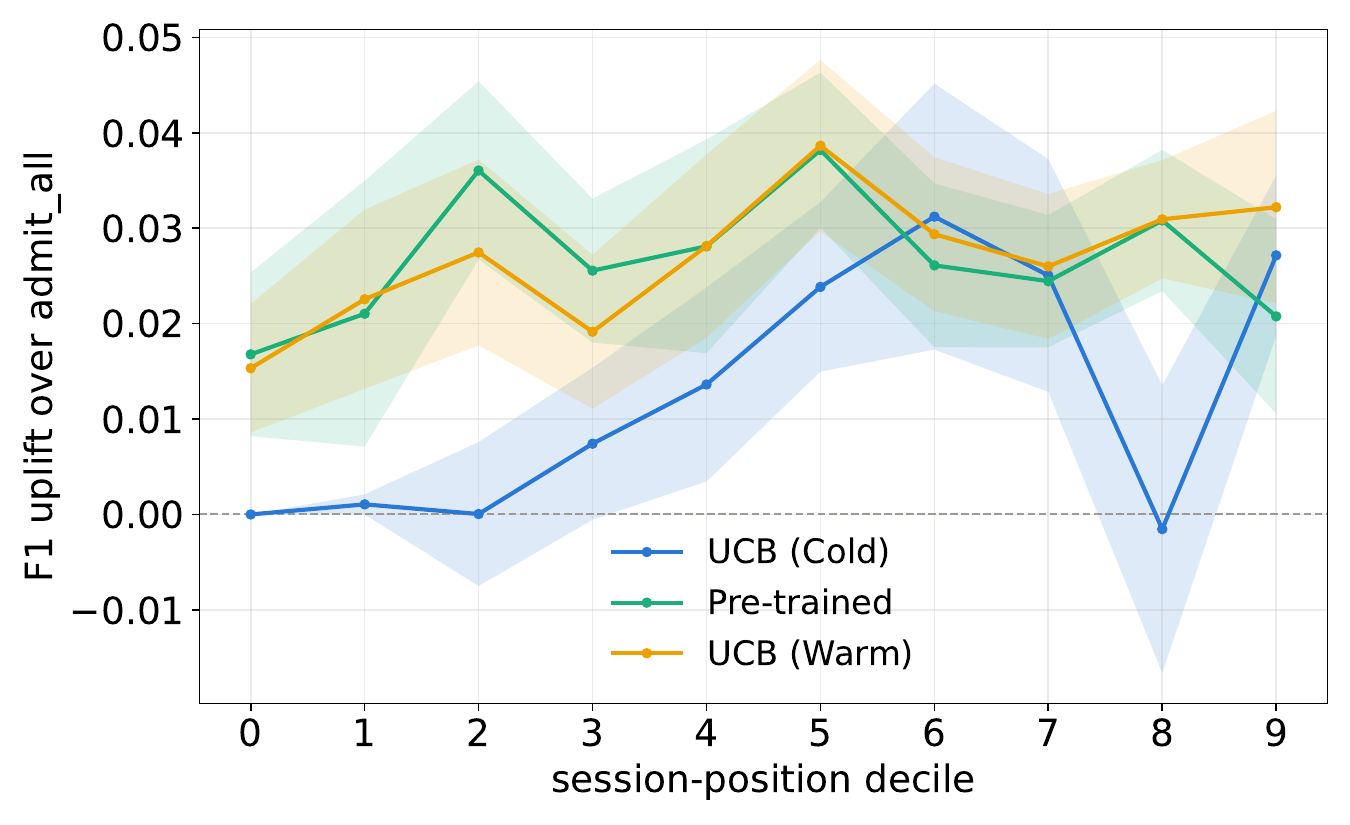}}
\subfigure[500]{\label{fig:short-500}\includegraphics[width=0.24\textwidth]{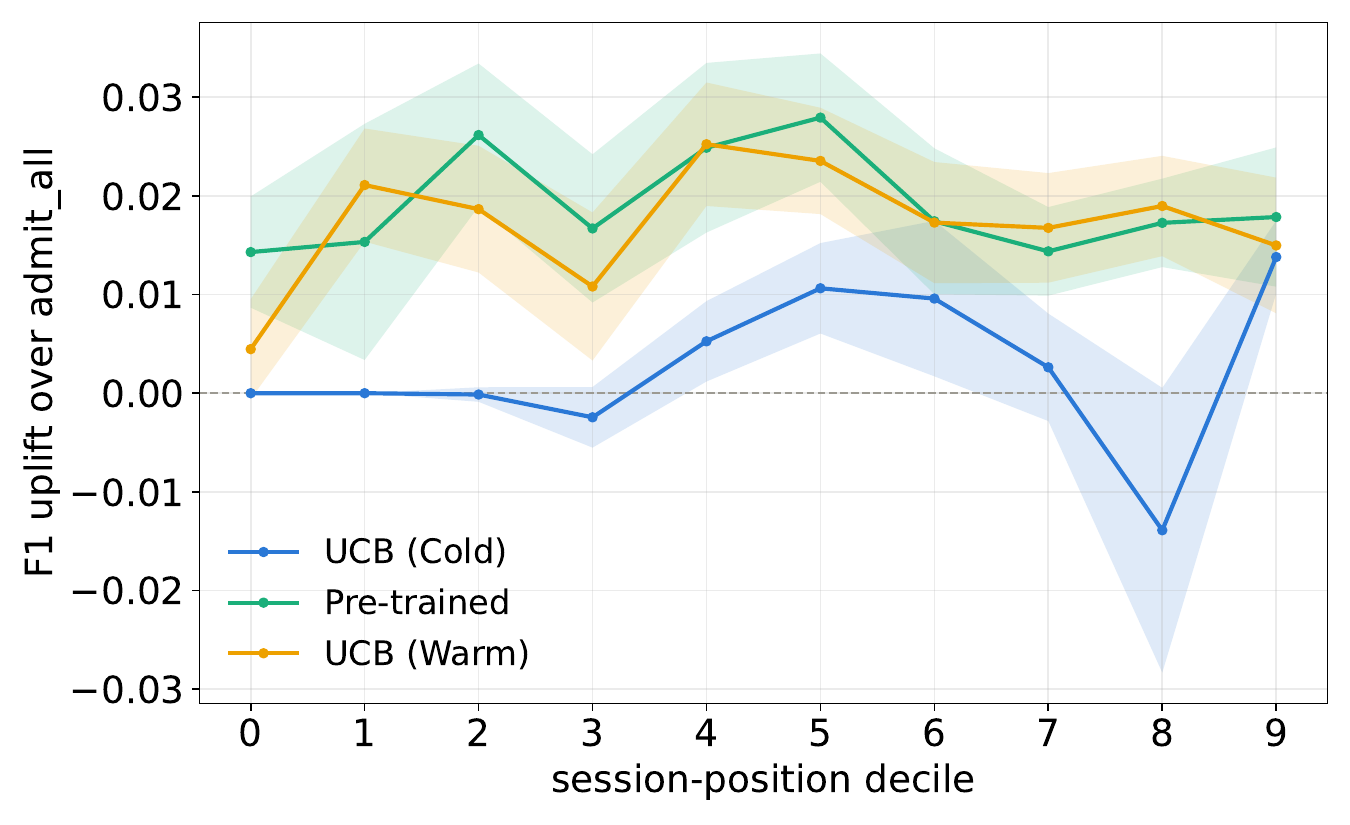}}
\caption{Learning dynamics on \datashort: F1 gain over the Universal baseline in each session-position decile under memory budgets $B \in \{60, 120, 250, 500\}$. Shaded regions denote the 95\% CI over all users.}
\label{fig:learning-curve-short}
\end{figure*}

\begin{figure*}[ht]
\centering   
\subfigure[60]{\label{fig:short-pt-60}\includegraphics[width=0.24\textwidth]{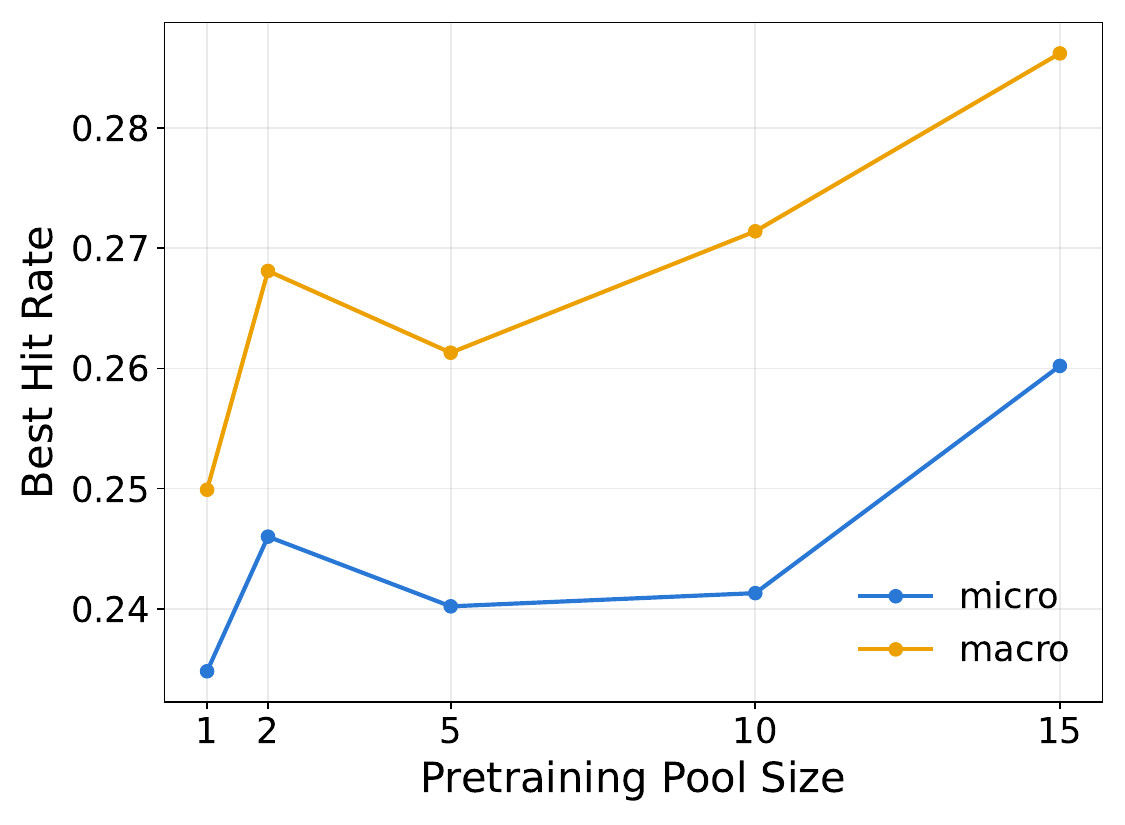}}
\subfigure[120]{\label{fig:short-pt-120}\includegraphics[width=0.24\textwidth]{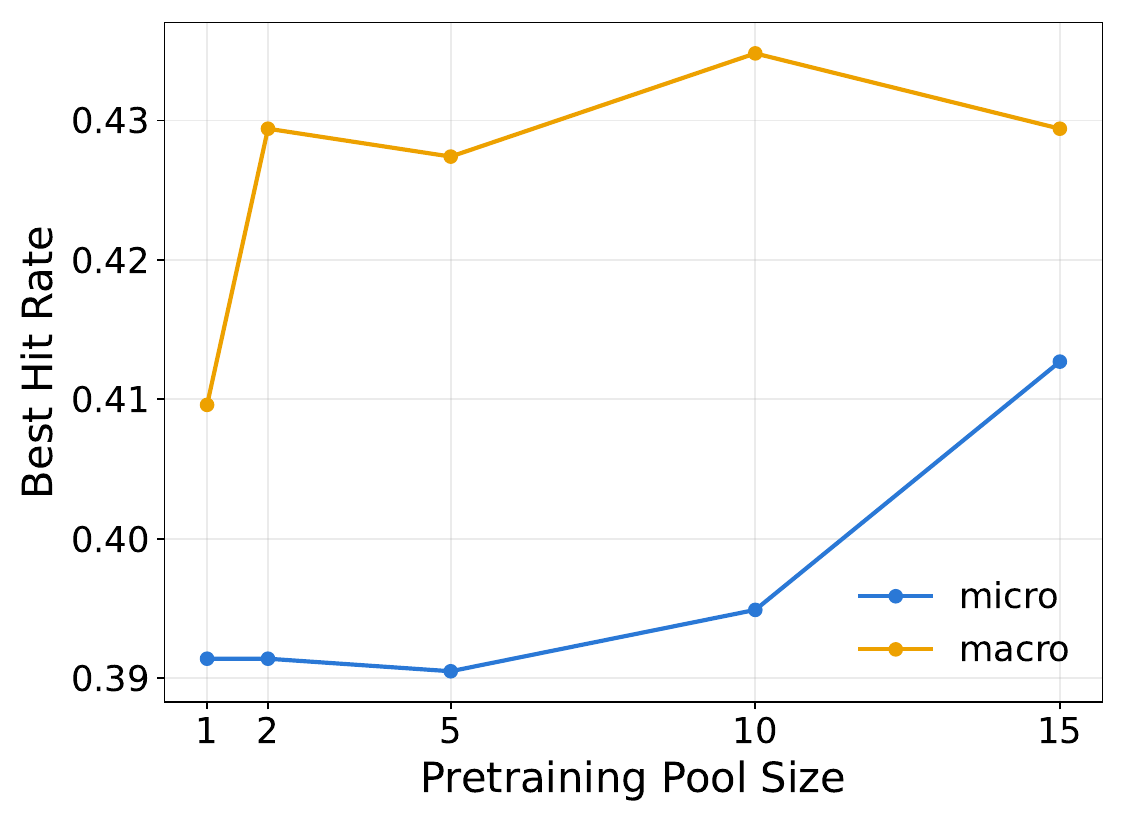}}
\subfigure[250]{\label{fig:short-pt-250}\includegraphics[width=0.24\textwidth]{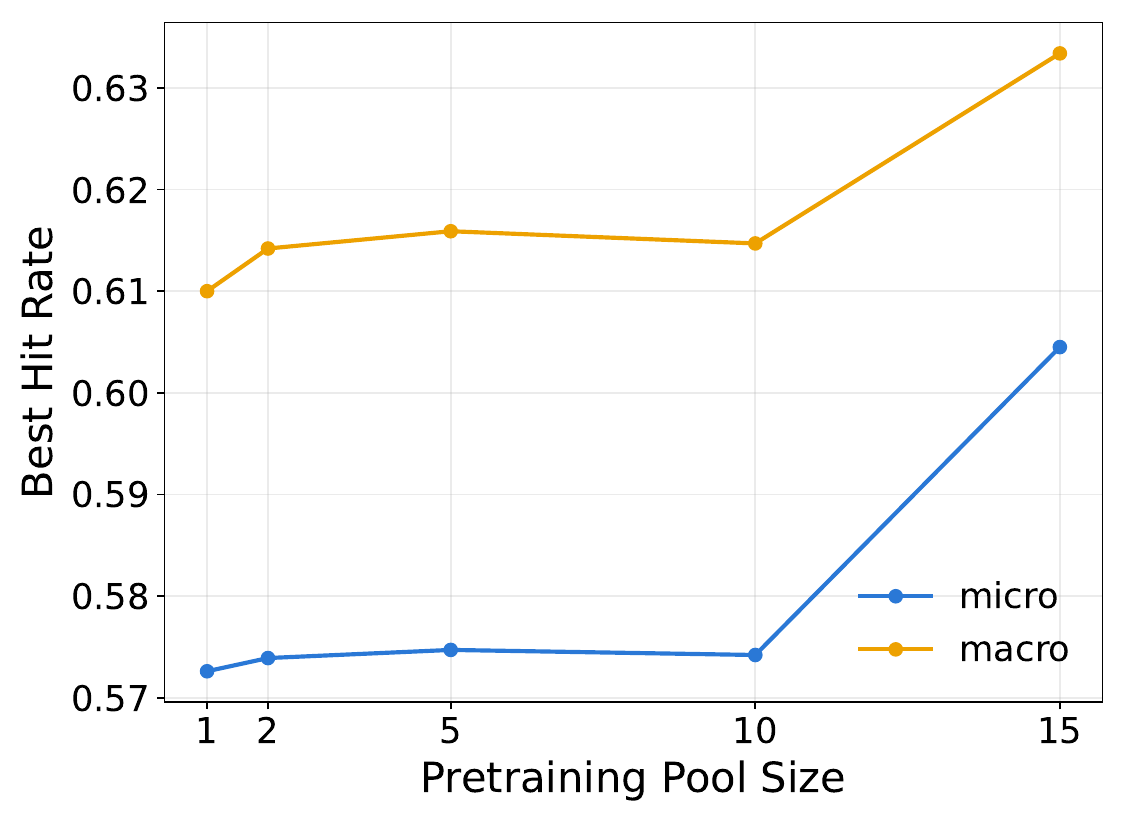}}
\subfigure[500]{\label{fig:short-pt-500}\includegraphics[width=0.24\textwidth]{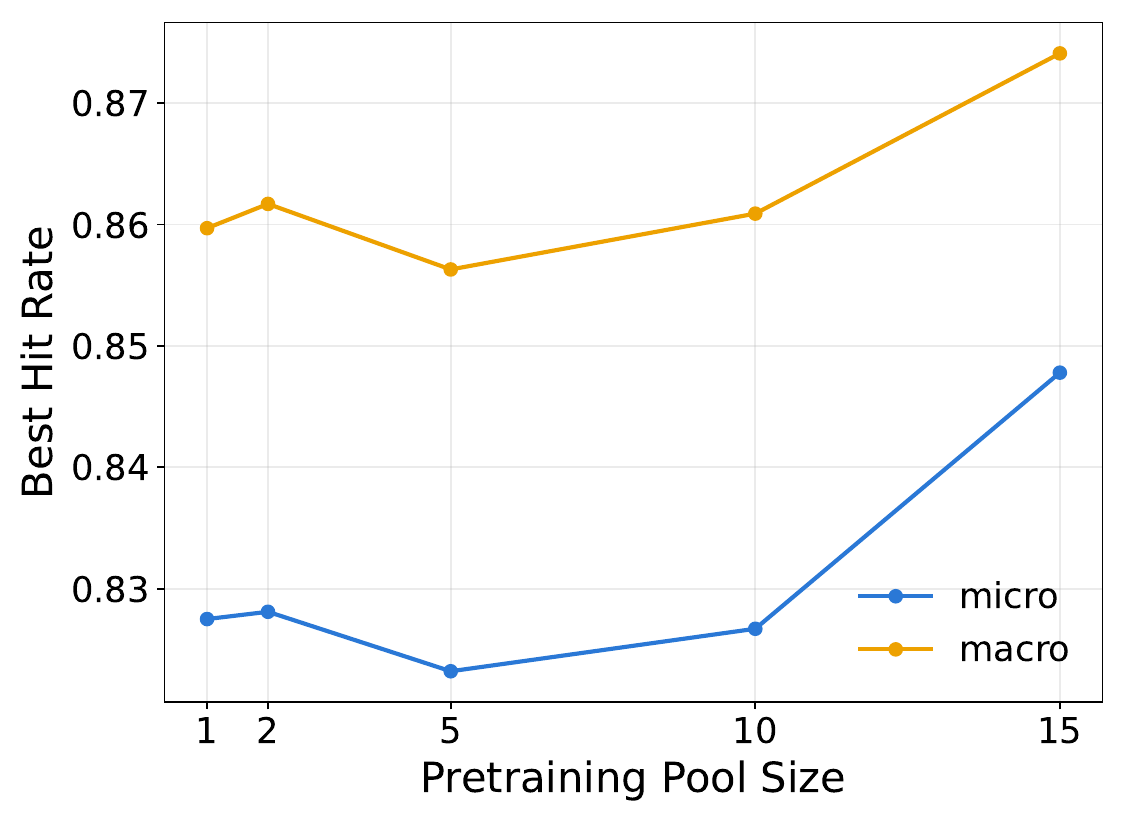}}
\caption{Effect of the pre-training pool size on \datashort: best micro- and macro-averaged RHRs under $B \in \{60, 120, 250, 500\}$.}
\label{fig:pretrain-size-curve-short-app}
\end{figure*}

On \datalong, the additional budgets reproduce the trends in Section~\ref{sec:ablation}. $\methodpretrained$ gives an immediate gain that shrinks in later deciles. Warm-Start UCB separates from it at later deciles. $\methodcoldstart$ remains close to the Universal baseline. RHRs also increase monotonically with the pool size under both budgets, improving by roughly $3$--$4$ points from one to fifteen pre-training users.
 
On \datashort, the same qualitative picture holds, and pre-training again gives a positive gain from the first decile. Warm-Start UCB and $\methodpretrained$ both outperform $\methodcoldstart$ for most of the interaction, and RHRs with fifteen pre-training users exceed those with a single user under every budget. The results are, however, noticeably noisier when the interaction history is short. In Figure~\ref{fig:learning-curve-short}, the confidence bands of Warm-Start UCB and $\methodpretrained$ largely overlap, and a clear separation between them appears only at $B = 60$. Cold-start UCB fluctuates strongly and drops sharply in the ninth decile under all four budgets before recovering. In Figure~\ref{fig:pretrain-size-curve-short-app}, performance does not increase monotonically with the pool size: it dips at a pool size of five for $B = 60$ and $B = 500$ and stays nearly flat between two and ten users for $B = 120$ and $B = 250$. This is because, with a substantially shorter interaction history, the online estimator does not collect sufficient feedback to converge, and the offline pre-training pool is also much smaller. We therefore regard \datalong as the more reliable testbed for studying learning dynamics, and we report it in the main text.
 

\section{Discussion of Assumption~\refwarmstartassump}
\label{app:discuss-assumption-pretrain}

In this section, we provide a two-step justification of Assumption~\ref{assumption:warm-start}. The first step controls the variation of the current user's parameter around a population center, while the second controls the error incurred when estimating this population center from other users' data. Let $\overline{\btheta}\coloneqq\EE_{\btheta^\star\sim\rho}[\btheta^\star]$ denote the population-level parameter. Suppose first that the current user's parameter is concentrated around $\overline{\btheta}$ in the geometry induced by $\bSigma_0$. In particular, assume that
\begin{align}
    \PP_{\btheta^\star\sim\rho}\Big(\|\btheta^\star-\overline{\btheta}\|_{\bSigma_0}\leq\epsilon_{\mathrm{pop}}\Big)\geq1-\delta_{\mathrm{pop}}.
    \label{eq:population-concentration-app}
\end{align}
The cross-user predictability observed in Section~\ref{sec:meta-analysis} provides empirical motivation for the existence of such a shared population component. At the same time, the gap between the within-user and leave-one-user-out performance suggests that $\epsilon_{\mathrm{pop}}$ need not be zero: it captures the irreducible heterogeneity between the current user and the population.


We next justify that the population center $\overline{\btheta}$ can be estimated accurately from other users' data. Suppose that the pretraining sample $\{(\xb_i,y_i)\}_{i=1}^n$ satisfies
\begin{align*}
    y_i=\langle\overline{\btheta},\xb_i\rangle+\xi_i, \qquad \EE[\xi_i\mid\xb_i]=0,
\end{align*}
where $\xi_i$ includes both observation noise and zero-mean variation among the pretraining users. Let
\begin{align*}
    \bSigma_x\coloneqq\EE[\xb\xb^\top], \qquad \widehat{\bSigma}_n\coloneqq\frac{1}{n}\sum_{i=1}^n\xb_i\xb_i^\top,
\end{align*}
and suppose $\bSigma_x\succeq\mu\Ib$ for some $\mu>0$. Consider the ridge estimator with regularization $\gamma_n$ to be determined later,
\begin{align*}
    \bmu_0\coloneqq\big(\widehat{\bSigma}_n+\gamma_n\Ib\big)^{-1}\frac{1}{n}\sum_{i=1}^n\xb_i y_i.
\end{align*}
Standard random-design ridge-regression results, such as \citet[Theorem~2]{hsu2014random}, imply that, when the sample size $n$ is sufficiently large\footnote{Roughly speaking, we require $n \gtrsim \gamma_n^{-1}\log(d/\delta_{\mathrm{est}})$. We refer to~\citet{hsu2014random} for the exact burn-in cost.}, with probability at least $1-\delta_{\mathrm{est}}$,
\begin{align}
    \|\bmu_0-\overline{\btheta}\|_{\bSigma_x}\lesssim\gamma_n\|\overline{\btheta}\|_{\bSigma_x^{-1}}+\sigma\sqrt{\frac{d_{\mathrm{eff}}(\gamma_n)+\log\big(1/\delta_{\mathrm{est}}\big)}{n}},
    \label{eq:pretraining-estimation}
\end{align}
up to lower-order terms arising from estimating the covariance matrix, where
\begin{align*}
    d_{\mathrm{eff}}(\gamma_n)\coloneqq\operatorname{tr}\Big(\bSigma_x\big(\bSigma_x+\gamma_n\Ib\big)^{-1}\Big)\leq d.
\end{align*}

For the idealized choice $\bSigma_0=\bSigma_x/\mu$, which satisfies $\bSigma_0\succeq\Ib$, \eqref{eq:pretraining-estimation} gives
\begin{align}
    \|\bmu_0-\overline{\btheta}\|_{\bSigma_0} \lesssim \frac{\gamma_n}{\sqrt\mu}\|\overline{\btheta}\|_{\bSigma_x^{-1}}+\frac{\sigma}{\sqrt\mu}\sqrt{\frac{d_{\mathrm{eff}}(\gamma_n)+\log\big(1/\delta_{\mathrm{est}}\big)}{n}} = \epsilon_{\mathrm{est}}.
    \label{eq:pretraining-estimation-sigma0}
\end{align}
with probability at least $1-\delta_{\mathrm{est}}$. More generally, the same conclusion holds when $\bSigma_0$ is an empirical covariance estimate that is spectrally equivalent to $\bSigma_x/\mu$. Moreover, when $\gamma_n$ is sufficiently small and $n$ sufficiently large, the ridge bias is dominated by the stochastic error, leading to a sufficiently small error $\epsilon_{\mathrm{est}} = \tilde{O} (\sigma \sqrt{d/(\mu n)})$.

On the intersection of the events in \eqref{eq:population-concentration-app} and \eqref{eq:pretraining-estimation-sigma0}, the triangle inequality gives
\begin{align*}
    \|\btheta^\star-\bmu_0\|_{\bSigma_0}&\leq\|\btheta^\star-\overline{\btheta}\|_{\bSigma_0}+\|\overline{\btheta}-\bmu_0\|_{\bSigma_0} \leq \epsilon_{\mathrm{pop}}+\epsilon_{\mathrm{est}}.
\end{align*}
Therefore, we recover Assumption~\ref{assumption:warm-start} by setting $\epsilon=\epsilon_{\mathrm{pop}}+\epsilon_{\mathrm{est}}$ and $\delta_0=\delta_{\mathrm{pop}}+\delta_{\mathrm{est}}$:
\begin{align*}
    \PP\Big(\|\btheta^\star-\bmu_0\|_{\bSigma_0}\leq\epsilon_{\mathrm{pop}}+\epsilon_{\mathrm{est}}\Big)\geq1-\delta_{\mathrm{pop}}-\delta_{\mathrm{est}}.
\end{align*}

\begin{remark}
    While the estimation component $\epsilon_{\mathrm{est}}$ decreases as the amount of pretraining data grows, the population component $\epsilon_{\mathrm{pop}}$ captures genuine user heterogeneity and generally does not vanish with $n$.
\end{remark}

\section{Missing Proof from Section~\refdelayedsec}

In this section, we provide the proof of Theorem~\ref{thm:main}. We first define the following filtration. For $t = 1,2, \dots, T$, we consider $\cF_t$ generated by
\begin{align*}
    \xb_1, a_1, y_1, D_1, \cdots, \xb_{t-1}, a_{t-1}, y_{t-1}, D_{t-1}, \xb_t, a_t
\end{align*}
where $D_s$ is defined as is if $a_s = \actionadd$ and otherwise $0$. For $t \geq 1$, we have $\xb_t \in \cF_{t}$, $y_t, D_t \in \cF_{t+1}$, and $\bar z_{s,t} = y_s \ind\{D_s \leq \min\{t-s,m\}\} \in \cF_{s+1}$.

\subsection{Proof of Optimism}

\begin{lemma}
\label{lem:delayed-optimistic-prior}
Let the warm-start event $\cE_0 \coloneqq \big\{\|\btheta^\star-\bmu_0\|_{\bSigma_0}\leq\epsilon\big\}$. Then conditional on $\cE_0$, with probability at least $1-\delta$, simultaneously for all $t\in[T]$,
\begin{align*}
    \|\widehat{\btheta}_t-\btheta_m^\star\|_{\Vb_t}\leq\alpha_t,
\end{align*}
where $\btheta_m^\star \coloneqq \kappa_m\btheta^\star$. Consequently, if $\PP(\cE_0)\geq1-\delta_0$, then the preceding event holds with probability at least $1-\delta_0-\delta$.

\end{lemma}

\begin{proof}[Proof of Lemma~\ref{lem:delayed-optimistic-prior}]
Let $\cS_t \coloneqq \big\{s\leq t:a_s=\actionadd\big\}$ denote the set of rounds up to time $t$ in which the learner plays action $\actionadd$. We partition $\cS_t$ into the completed and pending sets
\begin{align*}
    \cC_t \coloneqq \big\{s\in\cS_t:s\leq t-m\big\}, \qquad \cP_t \coloneqq \big\{s\in\cS_t:t-m<s\leq t\big\}.
\end{align*}
For every $s\in\cC_t$, the $m$-round observation window has elapsed, and hence
\begin{align*}
    \bar z_{s,t}=y_s\ind\{D_s\leq m\}\eqqcolon z_s^{(m)}.
\end{align*}
Because $\PP(D_s\leq m)=\kappa_m$ and $D_s$ is conditionally independent of $y_s$ given $\xb_s$, we have
\begin{align*}
    \EE\big[z_s^{(m)}\mid\xb_s\big]=\kappa_m\EE\big[y_s\mid\xb_s\big]=\kappa_m\eta(\xb_s)=\langle\btheta_m^\star,\xb_s\rangle.
\end{align*}
Therefore, defining $\xi_s \coloneqq z_s^{(m)}-\langle\btheta_m^\star,\xb_s\rangle$, we obtain a conditionally mean-zero and $\sigma$-sub-Gaussian noise sequence. 
Now define the completed-sample design matrix and its corresponding ridge estimator
\begin{align*}
    \overline{\Vb}_t \coloneqq \lambda\bSigma_0+\sum_{s\in\cC_t}\xb_s\xb_s^\top, \quad \overline{\btheta}_t \coloneqq \overline{\Vb}_t^{-1}\Bigg(\lambda\bSigma_0(\kappa_m\bmu_0)+\sum_{s\in\cC_t}\xb_s z_s^{(m)}\Bigg).
\end{align*}
Subtracting $\btheta_m^\star$ gives
\begin{align*}
    \overline{\Vb}_t\big(\overline{\btheta}_t-\btheta_m^\star\big)=\lambda\bSigma_0\big(\kappa_m\bmu_0-\btheta_m^\star\big)+\sum_{s\in\cC_t}\xb_s\xi_s.
\end{align*}
Since $\btheta_m^\star=\kappa_m\btheta^\star$, the prior-mismatch term satisfies, on $\cE_0$,
\begin{align*}
    \|\kappa_m\bmu_0-\btheta_m^\star\|_{\lambda\bSigma_0}=\kappa_m\sqrt{\lambda}\|\bmu_0-\btheta^\star\|_{\bSigma_0}\leq\kappa_m\sqrt{\lambda}\epsilon.
\end{align*}

We next bound the determinant term in the self-normalized inequality. Since $\xb_s \in \cF_s$ and $\xi_s \in \cF_{s+1}$ is $\sigma$-sub-Gaussian with $\sigma=1/2$ condition on $\cF_{s}$, we can apply Lemma~\ref{lem:self-normalized}, which leads to that with probability at least $1 - \delta$, for all $t \geq 0$
\begin{align*}
    \Big\| \sum_{s\in\cC_t} \xb_s\xi_s \Big\|_{\overline{\Vb}_t^{-1}}^2 \le 2\sigma^2 \log \bigg(\frac{\det(\overline{\Vb}_t)^{1/2}\det(\lambda\bSigma_0)^{-1/2}}{\delta} \bigg)
\end{align*}
Since $\bSigma_0\succeq\Ib$ and $\|\xb_s\|_2\leq L$, Lemma~\ref{lem:log-det-prior} leads to 
\begin{align*}
    \log\Bigg(\frac{\det(\overline{\Vb}_t)}{\det(\lambda\bSigma_0)}\Bigg)\leq d\log\Big(1+\frac{|\cC_t|L^2}{\lambda d}\Big)\leq d\log\Big(1+\frac{TL^2}{\lambda d}\Big).
\end{align*}
Combining these estimation therefore implies that, with probability at least $1-\delta$, simultaneously for all $t\in[T]$,
\begin{align*}
\Big\|\sum_{s\in\cC_t}\xb_s\xi_s\Big\|_{\overline{\Vb}_t^{-1}} \leq \sigma \sqrt{d\log\Big(1+\frac{TL^2}{\lambda d}\Big) + 2\log\Big(\frac{1}{\delta}\Big)}.
\end{align*}
Combining the preceding inequality with the prior-mismatch bound yields
\begin{align*}
    \|\overline{\btheta}_t-\btheta_m^\star\|_{\overline{\Vb}_t}\leq\beta_{\lambda,\delta} := \sigma \sqrt{d\log\Big(1+\frac{TL^2}{\lambda d}\Big) + 2\log\Big(\frac{1}{\delta}\Big)} + \kappa_m\sqrt{\lambda}\epsilon
\end{align*}
simultaneously for all $t\in[T]$. It remains to translate $\overline{\btheta}_t$ to $\hat \btheta_t$. By the definition of Algorithm~\ref{alg:delayed-linear-prior}, we know that $\widehat{\btheta}_t=\Vb_t^{-1}\bb_t$, which leads to 
\begin{align*}
    \Vb_t (\hat{\btheta}_t - \btheta_m^\star) & = \bb_t-\Vb_t\btheta_m^\star \\
    & = \lambda\bSigma_0(\kappa_m\bmu_0) + \sum_{s \in \cC_t} \xb_s\bar z_{s,t} + \sum_{s \in \cP_t} \xb_s\bar z_{s,t}  - \Big( \lambda\bSigma_0+\sum_{s\in\cC_t}\xb_s\xb_s^\top+\sum_{s\in\cP_t}\xb_s\xb_s^\top \Big) \btheta_m^\star.
\end{align*}
Since $\overline{\btheta}_t = \overline{\Vb}_t^{-1} \big(\lambda\bSigma_0(\kappa_m\bmu_0) + \sum_{s \in \cC_t} \xb_s\bar z_{s,t} \big)$, we know that
\begin{align*}
    \Vb_t (\hat{\btheta}_t - \btheta_m^\star) & = \overline{\Vb}_t (\overline{\btheta}_t - \btheta_m^\star) + \sum_{s \in \cP_t} \xb_s\bar z_{s,t} - \Big( \sum_{s\in\cP_t}\xb_s\xb_s^\top \Big) \btheta_m^\star \\
    & = \overline{\Vb}_t (\overline{\btheta}_t - \btheta_m^\star) + \sum_{s\in\cP_t}\xb_s\big(\bar z_{s,t}-\langle\btheta_m^\star,\xb_s\rangle\big).
\end{align*}
Therefore, the triangle inequality gives
\begin{align*}
    \|\widehat{\btheta}_t-\btheta_m^\star\|_{\Vb_t}&=\|\bb_t-\Vb_t\btheta_m^\star\|_{\Vb_t^{-1}} \leq \|\overline{\Vb}_t\big(\overline{\btheta}_t-\btheta_m^\star\big)\|_{\Vb_t^{-1}}+\Big\|\sum_{s\in\cP_t}\xb_s\Big(\bar z_{s,t}-\langle\btheta_m^\star,\xb_s\rangle\Big)\Big\|_{\Vb_t^{-1}}.
\end{align*}
Because $\Vb_t\succeq\overline{\Vb}_t$, we have $\Vb_t^{-1}\preceq\overline{\Vb}_t^{-1}$. Consequently,
\begin{align*}
    \|\overline{\Vb}_t\big(\overline{\btheta}_t-\btheta_m^\star\big)\|_{\Vb_t^{-1}}^2&=\big(\overline{\btheta}_t-\btheta_m^\star\big)^\top\overline{\Vb}_t\Vb_t^{-1}\overline{\Vb}_t\big(\overline{\btheta}_t-\btheta_m^\star\big) \leq\|\overline{\btheta}_t-\btheta_m^\star\|_{\overline{\Vb}_t}^2 \leq\beta_{\lambda,\delta}^2.
\end{align*}
For each $s\in\cP_t$, both $\bar z_{s,t}$ and $\langle\btheta_m^\star,\xb_s\rangle=\kappa_m\eta(\xb_s)$ belong to $[0,1]$. Therefore,
\begin{align*}
    \big|\bar z_{s,t}-\langle\btheta_m^\star,\xb_s\rangle\big|\leq1.
\end{align*}
The triangle inequality then gives
\begin{align*}
    \Big\|\sum_{s\in\cP_t}\xb_s\Big(\bar z_{s,t}-\langle\btheta_m^\star,\xb_s\rangle\Big)\Big\|_{\Vb_t^{-1}}&\leq\sum_{s\in\cP_t}\|\xb_s\|_{\Vb_t^{-1}}\big|\bar z_{s,t}-\langle\btheta_m^\star,\xb_s\rangle\big| \leq\sum_{s\in\cP_t}\|\xb_s\|_{\Vb_t^{-1}}.
\end{align*}
Combining the last three displays, we conclude that, simultaneously for all $t\in[T]$,
\begin{align*}
    \|\widehat{\btheta}_t-\btheta_m^\star\|_{\Vb_t}\leq\beta_{\lambda,\delta}+\sum_{s\in\cP_t}\|\xb_s\|_{\Vb_t^{-1}}=\alpha_t.
\end{align*}
This proves the conditional statement. Finally, combining the online concentration event with $\PP(\cE_0)\geq1-\delta_0$ proves the unconditional probability bound $1-\delta_0-\delta$.
\end{proof}

\subsection{Main Theorem}

\begin{proof}[Proof of Theorem~\ref{thm:main}]
Throughout the proof, we work on the good event $\cE$ from Lemma~\ref{lem:delayed-optimistic-prior}. On this event, simultaneously for all $t\in[T]$,
\begin{align*}
    \|\widehat{\btheta}_t-\btheta_m^\star\|_{\Vb_t}\leq\alpha_t=\beta_{\lambda,\delta}+\sum_{s\in\cP_t}\|\xb_s\|_{\Vb_t^{-1}},
\end{align*}
where $\btheta_m^\star=\kappa_m\btheta^\star$. We first show that the index $\kappa_m^{-1}\widetilde\eta_t$ is optimistic for $\eta(\xb_t)$. Fix any $t\in[T]$. By Cauchy--Schwarz and the definition of $\cE$,
\begin{align*}
    \langle\widehat{\btheta}_{t-1}-\btheta_m^\star,\xb_t\rangle&\geq-\|\widehat{\btheta}_{t-1}-\btheta_m^\star\|_{\Vb_{t-1}}\|\xb_t\|_{\Vb_{t-1}^{-1}} \geq -\alpha_{t-1}\|\xb_t\|_{\Vb_{t-1}^{-1}}.
\end{align*}
Recalling the computation of $\tilde \eta_t$, it follows that
\begin{align*}
    \widetilde\eta_t&=\langle\widehat{\btheta}_{t-1},\xb_t\rangle+\alpha_{t-1}\|\xb_t\|_{\Vb_{t-1}^{-1}} \geq\langle\btheta_m^\star,\xb_t\rangle =\kappa_m\eta(\xb_t).
    \tag{$\star$}
\end{align*}

Suppose first that $\eta(\xb_t)\geq\tau$. By $(\star)$, $\widetilde\eta_t\geq\kappa_m\eta(\xb_t)\geq\kappa_m\tau$, so Algorithm~\ref{alg:delayed-linear-prior} chooses $a_t=\mathsf{add}$, which agrees with the oracle policy. Thus, the algorithm never incorrectly skips a context that the oracle would add. Therefore, the only case in which the algorithm incurs nonzero regret is
\begin{align*}
    \eta(\xb_t)<\tau, \qquad a_t=\actionadd.
\end{align*}
In this case, the decision rule implies $\widetilde\eta_t\geq\kappa_m\tau$, and the per-step regret is given by
\begin{align*}
    \EE_y\big[r(\xb_t,\actionskip,y)\big]-\EE_y\big[r(\xb_t,\actionadd,y)\big]=\tau-\eta(\xb_t) \leq \kappa^{-1}_m \tilde \eta_t - \eta(\xb_t)
\end{align*}
Moreover, we see that
\begin{align*}
    \tilde \eta_t - \kappa_m \eta(\xb_t) & = \langle\widehat{\btheta}_{t-1},\xb_t\rangle + \alpha_{t-1}\|\xb_t\|_{\Vb_{t-1}^{-1}} - \langle\btheta_m^\star,\xb_t\rangle \\
    & = \langle\widehat{\btheta}_{t-1}-\btheta_m^\star,\xb_t\rangle+\alpha_{t-1}\|\xb_t\|_{\Vb_{t-1}^{-1}}\\
    &\leq2\alpha_{t-1}\|\xb_t\|_{\Vb_{t-1}^{-1}}.
\end{align*}
Consequently,
\begin{align*}
    \EE_y\big[r(\xb_t,\actionskip,y)\big]-\EE_y\big[r(\xb_t,\actionadd, y)\big] \leq\frac{2}{\kappa_m}\alpha_{t-1}\|\xb_t\|_{\Vb_{t-1}^{-1}}\ind\{a_t=1\}.
\end{align*}
Now let $\cS_T\coloneqq\big\{t\in[T]:a_t= \actionadd \big\}$ denote the set of rounds with action $\actionadd$. Summing the preceding inequality gives
\begin{align*}
    \regret(T)\leq\frac{2}{\kappa_m}\sum_{t\in\cS_T}\alpha_{t-1}\|\xb_t\|_{\Vb_{t-1}^{-1}}.
\end{align*}
Substituting the definition of $\alpha_{t-1}$, we decompose the sum as
\begin{align*}
    \sum_{t\in\cS_T}\alpha_{t-1}\|\xb_t\|_{\Vb_{t-1}^{-1}}&=\underbrace{\beta_{\lambda,\delta}\sum_{t\in\cS_T}\|\xb_t\|_{\Vb_{t-1}^{-1}}}_{\mathrm{(I)}} + \underbrace{\sum_{t\in\cS_T}\Bigg(\sum_{s\in\cP_{t-1}}\|\xb_s\|_{\Vb_{t-1}^{-1}}\Bigg)\|\xb_t\|_{\Vb_{t-1}^{-1}}}_{\mathrm{(II)}}.
\end{align*}

We first bound $\mathrm{(I)}$. Since $\lambda\geq L^2$ and $\bSigma_0\succeq\Ib$, we see that $\|\xb_t\|_{\Vb_{t-1}^{-1}}^2\leq\ \|\xb_t\|_2^2/\lambda \leq 1$. This allows us to apply the elliptical-potential lemma (Lemma~\ref{lem:pot}) to the subsequence of added contexts, which yields
\begin{align*}
    \sum_{t\in\cS_T}\|\xb_t\|_{\Vb_{t-1}^{-1}}^2&\leq2\log\Bigg(\frac{\det(\Vb_T)}{\det(\lambda\bSigma_0)}\Bigg) \leq 2d\log\Big(1+\frac{TL^2}{\lambda d}\Big),
\end{align*}
where the last inequality holds due to Lemma~\ref{lem:log-det-prior}. Therefore, Cauchy--Schwarz implies
\begin{align*}
    \sum_{t\in\cS_T}\|\xb_t\|_{\Vb_{t-1}^{-1}}&\leq\sqrt{|\cS_T|\sum_{t\in\cS_T}\|\xb_t\|_{\Vb_{t-1}^{-1}}^2} \leq\sqrt{2dT\log\Big(1+\frac{TL^2}{\lambda d}\Big)}.
\end{align*}
It follows that
\begin{align*}
    \mathrm{(I)}\leq\beta_{\lambda,\delta}\sqrt{2dT\log\Big(1+\frac{TL^2}{\lambda d}\Big)}.
\end{align*}

We next bound $\mathrm{(II)}$. Applying $ab\leq(a^2+b^2)/2$ termwise yields
\begin{align*}
    \mathrm{(II)}&\leq\frac{1}{2}\sum_{t\in\cS_T}\sum_{s\in\cP_{t-1}}\Big(\|\xb_s\|_{\Vb_{t-1}^{-1}}^2+\|\xb_t\|_{\Vb_{t-1}^{-1}}^2\Big) \eqqcolon\frac{1}{2}\big(\mathrm{(II)}_a+\mathrm{(II)}_b\big).
\end{align*}
For every $t$, the pending set $\cP_{t-1}$ contains at most $m$ elements. Hence,
\begin{align*}
    \mathrm{(II)}_b&\leq m\sum_{t\in\cS_T}\|\xb_t\|_{\Vb_{t-1}^{-1}}^2 \leq 2md\log\Big(1+\frac{TL^2}{\lambda d}\Big).
\end{align*}

To bound $\mathrm{(II)}_a$, we exchange the order of summation. Fix a round $s\in\cS_T$ such that $a_s = \actionadd$. The term involving $\xb_s$ appears only for rounds $t$ satisfying $s\in\cP_{t-1}$ and $a_t = \actionadd$, which requires $t\in\{s+1,\ldots,s+m\}$. Thus, it appears at most $m$ times. Moreover, for every such $t$, we have $t-1\geq s$ and therefore
\begin{align*}
    \Vb_{t-1}\succeq\Vb_s\succeq\Vb_{s-1} \, \Rightarrow \,  \|\xb_s\|_{\Vb_{t-1}^{-1}}^2\leq\|\xb_s\|_{\Vb_{s-1}^{-1}}^2.
\end{align*}
This immediately implies that
\begin{align*}
    \mathrm{(II)}_a &=\sum_{t\in\cS_T}\sum_{s\in\cP_{t-1}}\|\xb_s\|_{\Vb_{t-1}^{-1}}^2 \leq m\sum_{s\in\cS_T}\|\xb_s\|_{\Vb_{s-1}^{-1}}^2 \leq 2md\log\Big(1+\frac{TL^2}{\lambda d}\Big).
\end{align*}
Combining the bounds for $\mathrm{(II)}_a$ and $\mathrm{(II)}_b$ gives
\begin{align*}
    \mathrm{(II)}\leq2md\log\Big(1+\frac{TL^2}{\lambda d}\Big).
\end{align*}

Finally, combining the bounds for $\mathrm{(I)}$ and $\mathrm{(II)}$, we conclude that, on $\cE$,
\begin{align*}
    \regret(T)\leq\frac{2}{\kappa_m}\Bigg[\beta_{\lambda,\delta}\sqrt{2dT\log\Big(1+\frac{TL^2}{\lambda d}\Big)}+2md\log\Big(1+\frac{TL^2}{\lambda d}\Big)\Bigg].
\end{align*}
By Lemma~\ref{lem:delayed-optimistic-prior}, the event $\cE$ holds with probability at least $1-\delta$ conditional on the warm-start event $\cE_0$. Since $\PP(\cE_0)\geq1-\delta_0$, the claimed regret bound holds with probability at least $1-\delta_0-\delta$.

It remains to pick a proper $\lambda$. For notational convenience, define $\ell_\lambda\coloneqq \log\big(1+TL^2/(\lambda d)\big)$ and $q\coloneqq\log (1/\delta)$. We start with
\begin{align}
    \regret(T)\lesssim\inf_{\lambda\geq L^2}\Bigg\{\frac{1}{\kappa_m}\Big(\sigma\sqrt{d\ell_\lambda+2q}+\kappa_m\sqrt{\lambda}\epsilon\Big)\sqrt{Td\ell_\lambda}+\frac{md}{\kappa_m}\ell_\lambda\Bigg\}. \label{eq:final-regret-w-lambda}
\end{align}

We first derive the warm-start bound. Using $\log(1+u)\leq u$ gives $\ell_\lambda \leq TL^2/(\lambda d)$. Consequently, the warm-start contribution satisfies
\begin{align*}
    \sqrt{\lambda}\epsilon\sqrt{Td\ell_\lambda}\leq\sqrt{\lambda}\epsilon\sqrt{\frac{T^2L^2}{\lambda}}=LT\epsilon.
\end{align*}
Meanwhile, as $\lambda\to\infty$, we see that $\ell_\lambda \to 0$; thus, the remaining term vanishes:
\begin{align*}
    \frac{\sigma}{\kappa_m}\sqrt{d\ell_\lambda+2q}\sqrt{Td\ell_\lambda}\longrightarrow0, \qquad \frac{md}{\kappa_m}\ell_\lambda\longrightarrow0.
\end{align*}
Taking the infimum over $\lambda$ therefore yields
\begin{align*}
    \regret(T)\lesssim LT\epsilon = O(T\epsilon).
\end{align*}

We next derive the learning-based bound. Suppose $\epsilon>0$ and choose
\begin{align*}
    \lambda_{\mathrm{learn}}\coloneqq\max\Bigg\{L^2,\frac{d}{\kappa_m^2\epsilon^2}\Bigg\}.
\end{align*}
When $\epsilon$ is sufficiently small such that $\epsilon\lesssim\sqrt{d}/(\kappa_mL)$, this reduces to $\lambda_{\mathrm{learn}}=d\kappa_m^{-2}\epsilon^{-2}$. For this choice, we see that $\kappa_m\sqrt{\lambda_{\mathrm{learn}}}\epsilon=\sqrt d$ and $\ell_{\lambda_{\mathrm{learn}}}=\log\big(1+ \kappa_m^2\epsilon^2TL^2/d^2 \big) = \Tilde{O}(1)$.
Substituting this choice into~\eqref{eq:final-regret-w-lambda} gives
\begin{align*}
    \regret(T)&\lesssim\frac{\sigma}{\kappa_m}\sqrt{d\ell_{\lambda_{\mathrm{learn}}}+2q}\sqrt{Td\ell_{\lambda_{\mathrm{learn}}}}+\frac{d\sqrt{T\ell_{\lambda_{\mathrm{learn}}}}}{\kappa_m}+\frac{md}{\kappa_m}\ell_{\lambda_{\mathrm{learn}}}\\
    &\leq\widetilde O\Bigg(\frac{d\sqrt T+md}{\kappa_m}\Bigg), 
\end{align*}
where the last inequality suppresses logarithmic factors in $T$, $d$, $1/\delta$, and $1/\epsilon$. Combining this bound with the warm-start bound yields
\begin{align}
    \regret(T)\leq\widetilde O\Bigg(\min\Bigg\{T\epsilon,\frac{d\sqrt T+md}{\kappa_m}\Bigg\}\Bigg) \label{eq:final-two-phase-regret}
\end{align}
which gives the desired bound.

Finally, we characterize the phase transition so that $\lambda$ can be selected in a principled way. The boundary between the two bounds is determined by $T\epsilon\asymp d \kappa_m^{-1}\big(\sqrt T+m\big)$. Letting $u=\sqrt T$, the boundary equation becomes $\kappa_m\epsilon u^2=du+dm$. Solving this equation with respect to $u$ gives its positive solution
\begin{align*}
    u=\frac{d+\sqrt{d^2+4\kappa_m\epsilon dm}}{2\kappa_m\epsilon}.
\end{align*}
Therefore,
\begin{align*}
    T_{\mathrm{crit}}&\asymp\Bigg(\frac{d}{\kappa_m\epsilon}+\sqrt{\frac{dm}{\kappa_m\epsilon}}\Bigg)^2 \asymp \frac{d^2}{\kappa_m^2\epsilon^2}+\frac{dm}{\kappa_m\epsilon}.
\end{align*}
Therefore, setting $\lambda \to \infty$ when $T \leq T_{\mathrm{crit}}$ and $\lambda = d\kappa_m^{-2}\epsilon^{-2}$ leads to the rate in~\eqref{eq:final-two-phase-regret}.
\end{proof}


\section{Supplementary Lemmas}

\begin{lemma}[Lemma 11, \citealt{abbasi2011improved}]
\label{lem:pot}
Let $\xb_1,\ldots,\xb_t\in\RR^d$, and for $1\le s\le t$ define
\begin{align*}
\bar\Vb_s \;=\; \Vb + \sum_{r=1}^s \xb_r\xb_r^\top,
\end{align*}
where $\Vb\in\RR^{d\times d}$ is a (deterministic) positive definite matrix. Then
\begin{align*}
\log\left(\frac{\det(\bar\Vb_t)}{\det(\Vb)}\right) \;\le\; \sum_{s=1}^t \|\xb_s\|^2_{\bar\Vb_{s-1}^{-1}}.
\end{align*}
If, in addition, $\|\xb_s\|_2 \le L$ for all $s=1,\ldots,t$, then
\begin{align*}
\sum_{s=1}^t \min\left\{1,\, \|\xb_s\|^2_{\bar\Vb_{s-1}^{-1}}\right\} \;\le\; 2\log\left(\frac{\det(\bar\Vb_t)}{\det(\Vb)}\right).
\end{align*}
\end{lemma}

\begin{lemma}[Theorem 1, \citealt{abbasi2011improved}]
\label{lem:self-normalized}
Let $\{\mathcal F_t\}_{t=0}^\infty$ be a filtration. Let $\{\eta_t\}_{t=1}^\infty$ be a real-valued stochastic process such that $\eta_t$ is $\mathcal F_t$-measurable and, for some $\sigma\ge 0$, $\eta_t$ is conditionally $\sigma$-sub-Gaussian given $\mathcal F_{t-1}$, i.e.
\begin{align*}
\forall \lambda\in\RR,\qquad \EE\big[e^{\lambda\eta_t}\mid\mathcal F_{t-1}\big] \;\le\; \exp\!\Big(\frac{\lambda^2\sigma^2}{2}\Big) \qquad \text{a.s.}
\end{align*}
Let $\{\xb_t\}_{t=1}^\infty$ be an $\RR^d$-valued stochastic process such that $\xb_t$ is $\mathcal F_{t-1}$-measurable (i.e.\ $\xb_t$ may depend on the past but not on $\eta_t$). Let $\Vb\in\RR^{d\times d}$ be any deterministic positive definite matrix. For any $t\ge 0$, define
\begin{align*}
\bar\Vb_t \;=\; \Vb + \sum_{s=1}^t \xb_s\xb_s^\top, \qquad
S_t \;=\; \sum_{s=1}^t \eta_s\xb_s.
\end{align*}
Then, for any $\delta\in(0,1)$, with probability at least $1-\delta$, for all $t\ge 0$ simultaneously,
\begin{align*}
\big\|S_t\big\|_{\bar\Vb_t^{-1}}^2 \;\le\; 2\sigma^2\log\left(\frac{\det(\bar\Vb_t)^{1/2}\det(\Vb)^{-1/2}}{\delta}\right).
\end{align*}
\end{lemma}

The following lemma provides the calculation of the log-determinant term appearing in Lemma~\ref{lem:pot} and Lemma~\ref{lem:self-normalized}.

\begin{lemma}\label{lem:log-det-prior}
Let $\bSigma_0\succeq\lambda\Ib\succ0$, and let $\xb_1,\ldots,\xb_T\in\RR^d$ satisfy $\|\xb_t\|_2\leq L$ for every $t\in[T]$. Define $\bSigma_t\coloneqq\bSigma_0+\sum_{s=1}^t\xb_s\xb_s^\top$. Then, for every $t\in[T]$,
\begin{align*}
    \frac{\det(\bSigma_t)}{\det(\bSigma_0)}\leq\Big(1+\frac{tL^2}{d\lambda}\Big)^d\leq\Big(1+\frac{TL^2}{d\lambda}\Big)^d.
\end{align*}
\end{lemma}

\begin{proof}
Define
\begin{align*}
    \Mb_t\coloneqq\bSigma_0^{-1/2}\Bigg(\sum_{s=1}^t\xb_s\xb_s^\top\Bigg)\bSigma_0^{-1/2}\succeq0.
\end{align*}
Since $\bSigma_0\succeq\lambda\Ib$, we have $\bSigma_0^{-1}\preceq\lambda^{-1}\Ib$, and hence
\begin{align*}
    \operatorname{tr}(\Mb_t)&=\sum_{s=1}^t\xb_s^\top\bSigma_0^{-1}\xb_s \leq\frac{1}{\lambda}\sum_{s=1}^t\|\xb_s\|_2^2 \leq\frac{tL^2}{\lambda} \leq\frac{TL^2}{\lambda}.
\end{align*}
Let $\mu_1,\ldots,\mu_d$ be the eigenvalues of $\Mb_t$, then we see that
\begin{align*}
    \frac{\det(\bSigma_t)}{\det(\bSigma_0)}&=\det(\Ib+\Mb_t) = \prod_{i=1}^d(1+\mu_i).
\end{align*}
Applying AM--GM inequality,
\begin{align*}
    \frac{\det(\bSigma_t)}{\det(\bSigma_0)} = \prod_{i=1}^d(1+\mu_i) \leq\bigg(\frac{1}{d}\sum_{i=1}^d(1+\mu_i)\bigg)^d =\Big(1+\frac{\operatorname{tr}(\Mb_t)}{d}\Big)^d \leq\Big(1+\frac{TL^2}{d\lambda}\Big)^d,
\end{align*}
which finishes the proof.
\end{proof}

\end{document}